\PassOptionsToPackage{numbers, compress}{natbib}
\documentclass{article}

\usepackage[main, final]{neurips_2026}
\usepackage[utf8]{inputenc}
\usepackage[T1]{fontenc}

\usepackage{amsmath}
\usepackage{amssymb}
\usepackage{mathtools}
\usepackage{amsthm}

\usepackage{graphicx}
\usepackage{subcaption}
\usepackage{booktabs}
\usepackage{placeins}

\usepackage{xcolor}
\usepackage{microtype}
\usepackage{hyperref}

\usepackage{algorithm}
\usepackage{algorithmic}
\usepackage[capitalize,noabbrev]{cleveref}

\theoremstyle{plain}
\newtheorem{theorem}{Theorem}[section]

\newtheorem{lemma}[theorem]{Lemma}
\newtheorem{corollary}[theorem]{Corollary}
\theoremstyle{definition}

\newtheorem{assumption}[theorem]{Assumption}
\theoremstyle{remark}

\numberwithin{equation}{section}

\title{CRAFT: Conflict-Resolved Aggregation for \\ Federated Training}

\author{%
Ziqi Wang\\
Friedrich-Alexander-Universität\\
Erlangen-Nürnberg\\ 
\texttt{ziqi.wang@fau.de}
\And
Qiang Liu\\
School of CIT\\
Technical University of Munich\\
\texttt{qiang7.liu@tum.de}
\And
Nils Thuerey\\
Technical University of Munich\\
Munich Center for Machine Learning\\
\texttt{nils.thuerey@tum.de}
}

\begin{document}

\maketitle

\begin{abstract}
The aggregation of conflicting client updates remains a fundamental bottleneck in federated learning (FL) under heterogeneous data distributions. Naive averaging can produce a global update that improves the global objective while conflicting with specific clients, causing degradation for those clients. In this work, we propose CRAFT (Conflict-Resolved Aggregation for Federated Training), a new aggregation framework that treats the global update as a geometric correction problem. We formulate aggregation as finding the update closest to a reference direction while satisfying conflict-free alignment constraints. We derive a closed-form expression for the constrained optimization problem, avoiding the computational overhead of iterative solvers. Furthermore, we use a layer-wise adaptation to address conflicts at varying feature granularities. We provide a theoretical analysis showing that CRAFT promotes a common-descent structure and mitigates conflicts through its projection geometry. Extensive experiments on heterogeneous benchmarks demonstrate that CRAFT improves the accuracy of the global model while reducing performance disparity across clients compared with state-of-the-art baselines. The source code for CRAFT is available at \url{https://github.com/tum-pbs/CRAFT}.
\end{abstract}

\section{Introduction}
\label{sec:intro}

Federated learning (FL) \citep{mcmahan2017communicationefficient} enables collaborative training of machine learning models across decentralized devices, or clients, without sharing raw private data. Despite its privacy advantages, FL faces a critical challenge from heterogeneity, where client data are typically not independent and identically distributed (non-IID) and are often imbalanced \citep{kairouz2021advances}.

In practice, data heterogeneity induces divergent client updates. Conventional aggregation methods, such as FedAvg \citep{mcmahan2017communicationefficient}, compute weighted averages of local client updates. However, the averaged direction may benefit some clients while conflicting with others, thereby leading to substantial disparities in per-client accuracy \citep{li2019fair}. Consequently, beyond improving the overall mean accuracy, mitigating disparities in per-client performance is also of central importance.

Existing approaches address this challenge through several mechanisms. Heterogeneity-mitigating methods regularize or normalize local training, thereby limiting discrepancies between local and global updates \citep{li2020federated, wang2020fednova}. Meanwhile, fairness-aware methods either modify the federated objective or formulate FL as a multi-objective optimization problem that jointly accounts for client losses, with the goal of improving worst-client performance \citep{hu2023federatedlearningmeetsmultiobjective,mohri2019agnostic}. In addition, personalized FL methods improve client-specific performance by maintaining individualized models for each client \citep{li2021ditto,t2020personalized,zhang2023fedala}. More recently, conflict-aware methods have directly manipulated the geometry of client updates to reduce conflicts through vector projection or geometric conflict resolution \citep{liu2024config,pan2023fedmdfg,wang2021federatedlearningfairaveraging}. However, their performance can still be limited by heuristic conflict-resolution rules.

In this paper, we propose CRAFT (\textbf{C}onflict-\textbf{R}esolved \textbf{A}ggregation for \textbf{F}ederated \textbf{T}raining). Instead of averaging, CRAFT formulates aggregation as a reference-anchored constrained least-squares problem that finds the global update $\mathbf{g}$ closest to a prespecified reference direction while satisfying conflict-free alignment constraints $\langle \mathbf{g}, \mathbf{g}_i \rangle > 0$ for all clients $i$. Our key innovations are as follows.
\begin{itemize}

\item \textbf{Conflict-free aggregation by problem reformulation.}
We reformulate federated aggregation as a reference-anchored constrained least-squares problem rather than a weighted averaging rule. This formulation directly enforces positive alignment between the global update and every participating client update, yielding a conflict-free solution whenever the conflict-free alignment constraints are feasible. We further derive an explicit closed-form solution based on a Moore--Penrose correction, avoiding iterative projection or quadratic-programming solvers and keeping the server-side computation lightweight.

\item \textbf{Momentum-like reference and layer-wise adaptation.}
We use the previous-round global update as the reference direction, allowing CRAFT to preserve useful temporal information while correcting the components that conflict with the current active clients. We further introduce a layer-wise implementation that resolves conflicts at the granularity of model layers, making the method suitable for deep neural networks. The accompanying theoretical analysis characterizes the projection geometry, establishes its common-descent structure, and provides finite-time convergence guarantees.
\item \textbf{Strong empirical performance and plug-in extensibility.}
Extensive experiments on heterogeneous datasets, deep models, and large client populations show that CRAFT improves mean accuracy and, when mean accuracy is comparable, further enhances client-level fairness by reducing accuracy disparity and improving tail-client performance. CRAFT is designed for training a shared global model rather than client-specific personalized models. Nevertheless, it can serve as a generic aggregation operator that replaces weighted averaging in existing personalized FL pipelines, where this simple substitution already yields substantial performance gains and demonstrates the broad extensibility of CRAFT.
\end{itemize}

\begin{figure}[t]
\centering
\begin{subfigure}{0.49\linewidth}
    \centering
    \includegraphics[width=0.95\linewidth]{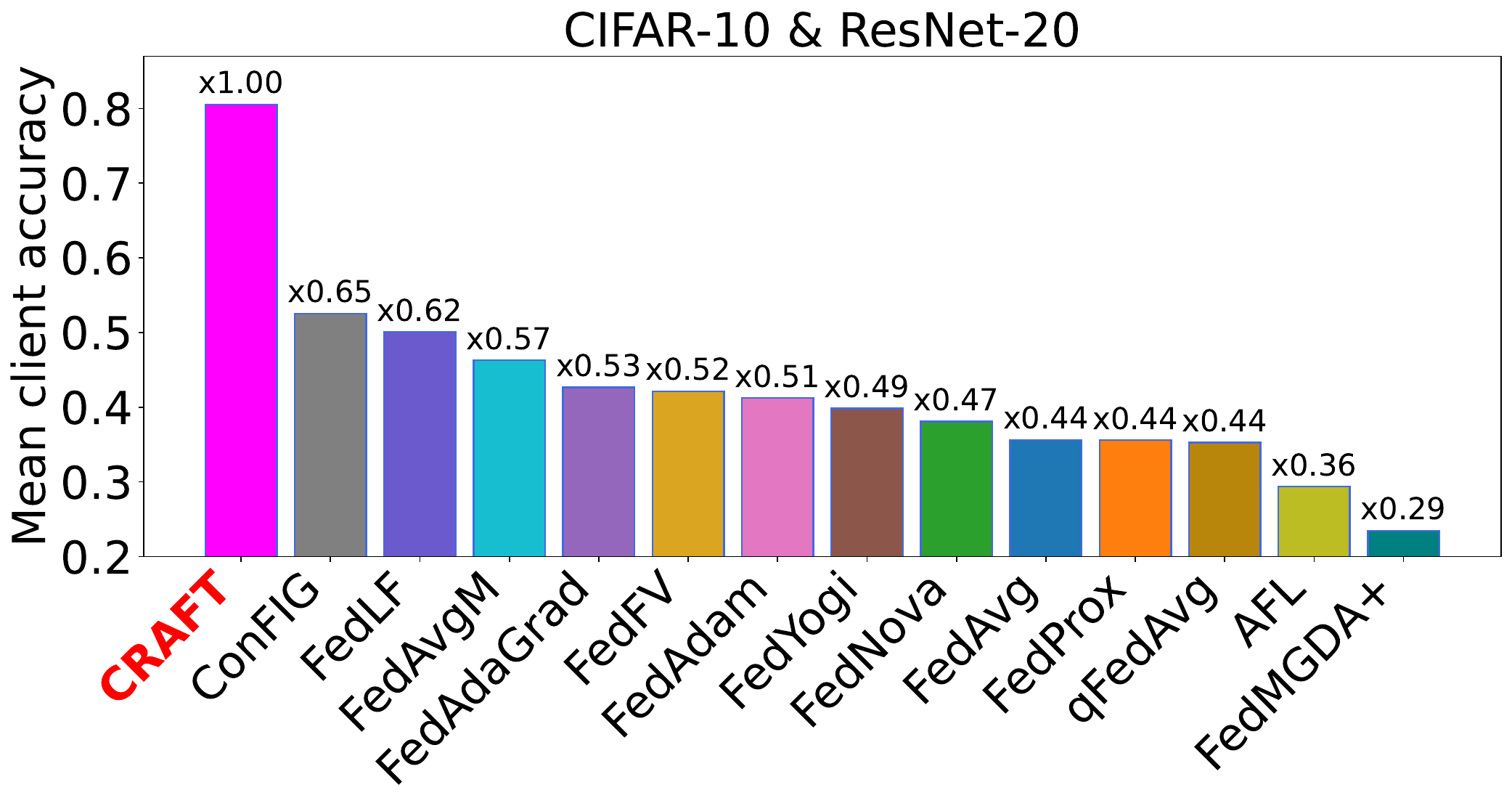}
    \label{fig:teaser_average_accuracy}
\end{subfigure}
\begin{subfigure}{0.49\linewidth}
    \centering
    \includegraphics[width=0.95\linewidth]{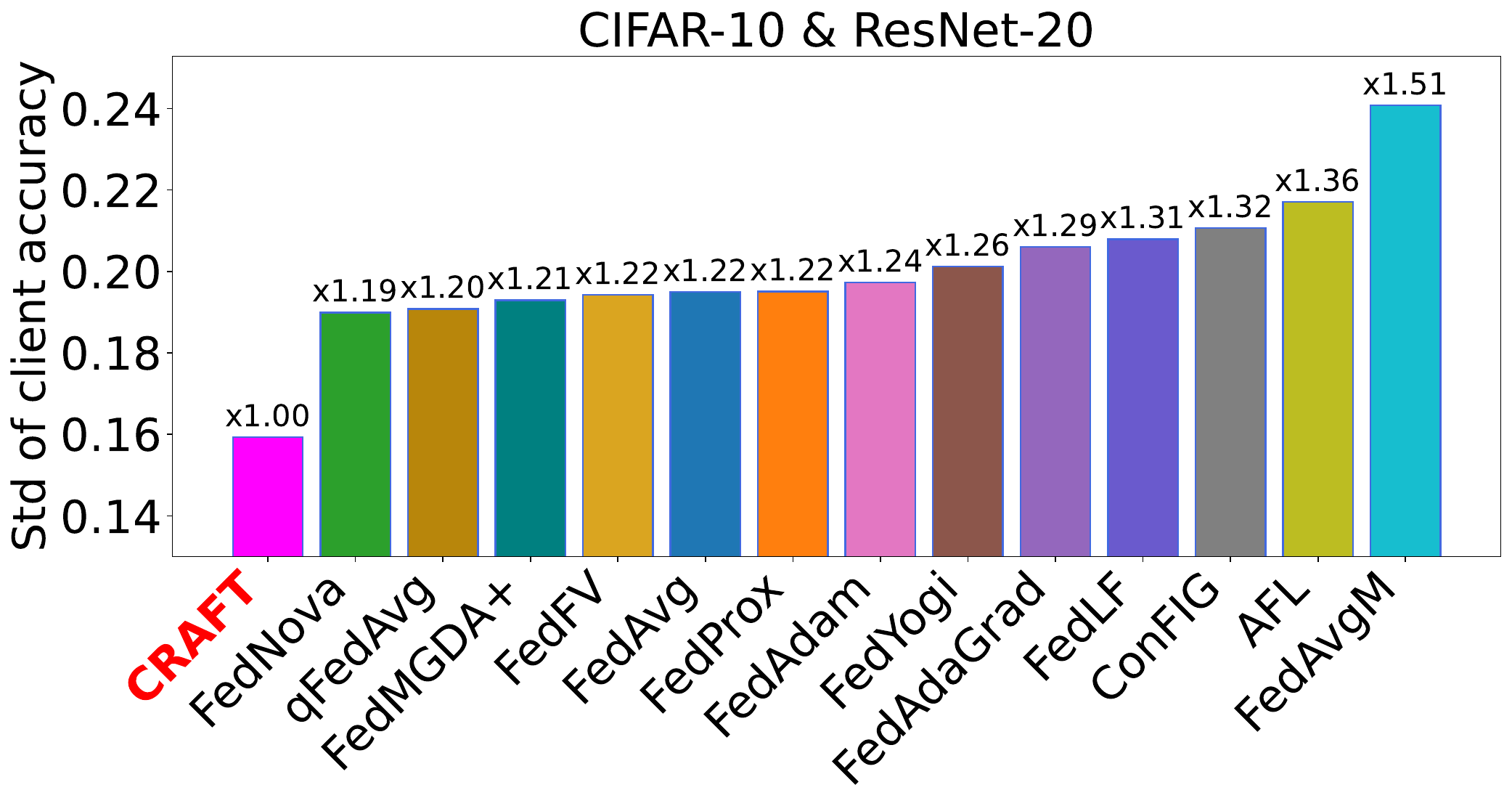}
    \label{fig:teaser_std_accuracy}
\end{subfigure}
\caption{Mean (left, $\uparrow$) and standard deviation (right, $\downarrow$) of per-client accuracy. CRAFT achieves higher mean client accuracy while simultaneously reducing client-level disparity.}
\end{figure}

\section{Related work}
\label{sec:related_work}

\textbf{Heterogeneity and fairness in FL.} Statistical and system heterogeneity are central challenges in FL \citep{kairouz2021advances}. FedAvg often struggles with non-IID data, leading to slow convergence or biased global models. FedProx \citep{li2020federated} proposes regularizing local training to keep updates closer to the global model. FedNova \citep{wang2020fednova} proposes normalizing client updates to address objective inconsistency caused by varying local training epochs. Adaptive server optimizers, such as FedAvgM \citep{hsu2019measuring} and FedOpt \citep{reddi2021adaptive}, have also been explored to improve convergence under heterogeneity \citep{kairouz2021advances}. These methods primarily target global optimization performance, but high mean accuracy does not guarantee reliable accuracy for every client. Client-level fairness aims to reduce performance disparity across clients \citep{li2019fair}. To this end, \citet{mohri2019agnostic} optimizes a worst-case objective via minimax optimization, and \citet{li2019fair} proposes q-FedAvg to reweight the federated objective in favor of lower-performing clients. These approaches address fairness by adjusting the objective function. By contrast, CRAFT tackles the fairness problem directly at the model-update level through conflict-free alignment constraints.

\textbf{Conflict mitigation.} In FL, FedFV \cite{wang2021federatedlearningfairaveraging} identifies gradient conflicts as a key source of unfairness and proposes iteratively removing conflicting components before averaging. FedLF~\cite{pan2024fedlf} and FedMDFG~\cite{pan2023fedmdfg} further develop this direction through layer-wise constraints or fairness-guided gradient adjustment. Similar gradient-conflict issues also arise in multi-task learning (MTL) \cite{smith2017federated, yu2020gradient}, continual learning (CL) \cite{Mehrdad2020,Riemer2019}, and physics-informed neural networks (PINNs) \cite{hwang2024dual,liu2024config}. For training PINNs, ConFIG~\cite{liu2024config} formulates conflict resolution as solving a linear system. However, it does not exploit round-wise information to guide the resolution process. To address this limitation, CRAFT formulates conflict resolution as a constrained least-squares problem with a momentum-like reference direction, yielding a more consistent and effective solution.

\textbf{Multi-objective optimization.}
In multi-objective formulations of FL, each client loss is treated as an objective rather than being absorbed into a single global objective. This is useful for fairness because improving one client should not come at the cost of degrading others. AFL~\citep{mohri2019agnostic} follows this idea by solving a minimax problem over model parameters and client weights, where the inner maximization selects the loss weighting that emphasizes clients with larger losses. FedMGDA+~\citep{hu2023federatedlearningmeetsmultiobjective} makes the multi-objective formulation explicit by treating each client loss as a separate objective and computing a minimum-norm convex combination of client gradients to obtain a common descent direction. These methods provide principled fairness-oriented objectives, but the aggregate direction is reconstructed from the current client gradients and may become small or unstable under strong heterogeneity. In contrast, CRAFT anchors aggregation to a momentum-like reference and applies the minimum correction needed to enforce positive alignment with every participating client.

\textbf{Relation to personalized FL.} Personalization methods handle heterogeneity by maintaining client-specific models alongside a shared model. Per-FedAvg~\cite{fallah2020personalized} learns an initial shared model that can be quickly adapted to each client through a few local gradient steps. FedRep~\cite{collins2021exploiting} decomposes the model into a shared backbone and client-specific heads, enabling personalization by locally fine-tuning the client-specific components. Ditto~\citep{li2021ditto}, instead of separating model components, learns an additional personalized model for each client by regularizing it toward the shared global model with a proximal term. CRAFT targets training a single global model, where the server outputs one shared model for all clients. Nevertheless, its conflict-free aggregation mechanism can be readily integrated into existing personalized FL frameworks, as demonstrated in the experiments. Therefore, CRAFT provides a more robust foundation that complements, rather than competes with, personalization strategies.

\section{Problem setup and preliminaries}
\label{sec:problem_setup}
\subsection{Federated optimization}
FL training with $N$ clients can be formulated as minimizing a weighted sum of local objectives:
\begin{equation}
\min_{\boldsymbol{\theta} \in \mathbb{R}^d} 
F(\boldsymbol{\theta}) \coloneqq \sum_{i=1}^N \rho_i f_i(\boldsymbol{\theta}),
\label{eq:global_obj}
\end{equation}
where client $i \in [N] \coloneqq \{1,2,\ldots, N\}$ holds a local dataset $\mathcal{D}_i$ and minimizes $f_i(\boldsymbol{\theta}) \coloneqq \mathbb{E}_{\xi \sim \mathcal{D}_i}[\ell(\boldsymbol{\theta}; \xi)]$, with $\boldsymbol{\theta} \in \mathbb{R}^d$ denoting the model parameters and $\xi$ denoting a random sample. The weights $\rho_i$ are typically chosen as $\rho_i = |\mathcal{D}_i| / \sum_{j=1}^N |\mathcal{D}_j|$.

FL training proceeds in communication rounds. At each round $t = 0, 1, \ldots, T-1$, the server broadcasts the current global model $\boldsymbol{\theta}_t$ to a subset of participating clients $\mathcal{S}_t$ for local training.

\textbf{Client-side.} Each client $i \in \mathcal{S}_t$ initializes its local model as $\boldsymbol{\theta}_i^{t,0} = \boldsymbol{\theta}_t$ and performs $K_i$ steps of local stochastic gradient descent. For $k = 0, \dots, K_i - 1$, the update is given by
\begin{equation}
\boldsymbol{\theta}_i^{t,k+1} 
=
\boldsymbol{\theta}_i^{t,k} 
- 
\eta_l \nabla f_i(\boldsymbol{\theta}_i^{t,k}; \xi_i^{t,k}),
\label{eq:local_sgd}
\end{equation}
where $\eta_l$ is the local learning rate and $\xi_i^{t,k}$ denotes a mini-batch sampled from $\mathcal{D}_i$. After local training, client $i$ sends the following accumulated update to the server for aggregation.
\begin{equation}
\mathbf{g}_i^t \coloneqq \boldsymbol{\theta}_t - \boldsymbol{\theta}_i^{t,K_i}.
\label{eq:client_update}
\end{equation}

\textbf{Server-side.} Upon receiving updates from participating clients, the server aggregates them to form a global update direction $\mathbf{g}_t$ and updates the global model as
\begin{equation}
\boldsymbol{\theta}_{t+1} 
=
\boldsymbol{\theta}_t 
- 
\eta_g \mathbf{g}_t.
\label{eq:server_update}
\end{equation}
Here, $\eta_g$ denotes the server learning rate. In FedAvg, the aggregation is given by
$
\mathbf{g}_t = \sum_{i \in \mathcal{S}_t} \rho_i^t \mathbf{g}_i^t
$, where $\rho_i^t$ denotes the data-proportional weights renormalized over $\mathcal{S}_t$.

\subsection{Conflict definition and fairness issue}
\label{subsec:conflict_definition}

\textbf{Conflict definition.} While the global FL objective \eqref{eq:global_obj} provides a unified optimization target, it does not explicitly account for conflicts between the aggregate update and individual client updates. Let $\mathbf{g}_i^t$ denote the update of client $i$ at round $t$. We say that a \emph{conflict} occurs between the aggregate update $\mathbf{g}_t$ and client $i$ if $\langle \mathbf{g}_i^t, \mathbf{g}_t \rangle < 0$. Since $\mathbf{g}_i^t$ approximates a scaled local gradient, the global update changes the local objective by approximately $-\eta_g\langle \nabla f_i(\boldsymbol{\theta}_t),\mathbf{g}_t\rangle$. Thus, a negative inner product indicates that the global update can increase client $i$'s loss to first order and often reduce its test accuracy.

\textbf{Fairness issue.} Conflicts create a fairness issue, since under data heterogeneity, averaging can improve mean accuracy by favoring clients aligned with the aggregate update while degrading clients in conflicting directions. This can lead to a large standard deviation of per-client accuracy. We therefore treat mean accuracy as a prerequisite, since a small standard deviation is uninformative if all clients remain at low accuracy. Given comparable mean accuracy, better fairness means reducing this standard deviation, \emph{i.e.}, narrowing the accuracy gap between head and tail clients.

\section{The CRAFT framework}
\label{sec:method}

In this section, we introduce the CRAFT framework. We begin by establishing the notation and intuition for conflict-free updates, then present our detailed CRAFT formulation.

\subsection{Conflict-free update}
Our approach is inspired by recent advances in MTL and PINNs, especially the ConFIG algorithm \citep{liu2024config}. Define the normalization operator $\mathcal{U}(\mathbf{v}) := \mathbf{v}/(\|\mathbf{v}\|+\varepsilon)$ with $\varepsilon>0$. Given $m$ active clients, let $\mathbf{U}_t = [\mathbf{u}_1^t, \dots, \mathbf{u}_m^t]^\top \in \mathbb{R}^{m \times d}$ be the normalized update matrix, with $\mathbf{u}_i^t = \mathcal{U}(\mathbf{g}_i^t)$ denoting the normalized local update of client $i$. To ensure positive alignment between the global update $\mathbf{g}$ and all participating client updates, we impose the following inner-product constraints:
\begin{equation}
\mathbf{U}_t \mathbf{g} = \boldsymbol{\rho}_t,
\label{eq:linear_constraints}
\end{equation}
where $\boldsymbol{\rho}_t = [\rho_1^t, \dots, \rho_m^t]^\top$ collects the positive round-$t$ targets for active clients. The design choices of targets are elaborated in \Cref{app:rho_choice}. Since $\rho_i^t>0$ for all $i$, any feasible solution of \eqref{eq:linear_constraints} is conflict-free with respect to all client updates. However, \eqref{eq:linear_constraints} does not uniquely determine the update direction. When $\mathbf{U}_t$ has full row rank, the system is feasible for any target $\boldsymbol{\rho}_t\in\mathbb{R}^m$. In federated learning, it is common for the model dimension $d$ to be much larger than the number of participating clients $m$, \emph{i.e.}, $d \gg m$. Thus, the system is underdetermined and admits infinitely many feasible solutions. We discuss this feasibility condition in \Cref{app:feasibility_discussion}.

Specifically, ConFIG chooses the feasible direction with the minimum Euclidean norm. Let $\mathbf{A}^{\dagger}$ denote the Moore--Penrose inverse of a matrix $\mathbf{A}$. For a consistent linear system, the pseudoinverse returns the unique minimum-norm solution. Applying this property to \eqref{eq:linear_constraints}, ConFIG returns $\mathbf{g}_{t}=\mathbf{U}_t^\dagger\boldsymbol{\rho}_t$, which is equivalently the minimizer of
\begin{equation}
\min_{\mathbf{g} \in \mathbb{R}^d} \; \frac{1}{2}\|\mathbf{g}\|^2 
\quad \text{s.t.} \quad 
\mathbf{U}_t \mathbf{g} = \boldsymbol{\rho}_t.
\label{eq:config_opt}
\end{equation}
This construction produces a conflict-free direction, but it does not take historical information into account. In FL, consecutive communication rounds are coupled through the evolving global model, and the previous aggregate direction carries useful information about the server's optimization trajectory. Reconstructing the update from scratch at every round can discard this historical signal, attenuate the update norm, and introduce unnecessary variation under partial participation. This motivates a different selection principle that maintains the same conflict-free alignment constraints while selecting the feasible update closest to a stable reference direction.

\subsection{Improved conflict resolution via correction}
Building on the above observation, CRAFT formulates aggregation as a \emph{correction} problem. Instead of constructing an update direction from the origin, we seek an update $\mathbf{g}$ that remains as close as possible to a reference direction $\hat{\mathbf{g}}_t$ while satisfying the same conflict-free alignment constraints.
\begin{equation}
\min_{\mathbf{g} \in \mathbb{R}^d} \frac{1}{2} \| \mathbf{g} - \hat{\mathbf{g}}_t \|^2 
\quad \text{s.t.} \quad 
\mathbf{U}_t \mathbf{g} = \boldsymbol{\rho}_t.
\label{eq:craft_opt}
\end{equation}
Let $\boldsymbol{\Delta} \coloneqq \mathbf{g}-\hat{\mathbf{g}}_t$ denote the \emph{correction} applied to the reference direction. Then \eqref{eq:craft_opt} is equivalently written as the following minimum-norm correction problem.
\begin{equation}
\min_{\boldsymbol{\Delta}\in\mathbb{R}^d}
\frac{1}{2}\|\boldsymbol{\Delta}\|^2
\quad
\text{s.t.}
\quad
\mathbf{U}_t\boldsymbol{\Delta}
=
\boldsymbol{\rho}_t-\mathbf{U}_t\hat{\mathbf{g}}_t.
\label{eq:craft_correction_opt}
\end{equation}
This is the same pseudoinverse principle applied to the correction variable. When the correction constraint in \eqref{eq:craft_correction_opt} is consistent, the resulting CRAFT update is
\begin{equation}
\mathbf{g}_t
=
\hat{\mathbf{g}}_t+\boldsymbol{\Delta}_t,
\quad \text{with} \quad
\boldsymbol{\Delta}_t
=
\mathbf{U}_t^\dagger
\bigl(\boldsymbol{\rho}_t-\mathbf{U}_t\hat{\mathbf{g}}_t\bigr).
\label{eq:craft_solution}
\end{equation}
Thus, CRAFT keeps the reference direction whenever it already satisfies the desired alignments and otherwise adds only the smallest Euclidean correction needed to enforce them. Since this correction lies in the row space of $\mathbf{U}_t$, the update modifies $\hat{\mathbf{g}}_t$ only along directions spanned by the participating client gradients, yielding the closest conflict-free update to the reference direction.

\subsection{Choice of reference direction}
We propose using the normalized \emph{global update} from the previous round as the reference direction.
\begin{equation}
\hat{\mathbf{g}}_t
=
\mathcal{U}(\mathbf{g}_{t-1}).
\label{eq:ref_update}
\end{equation}
At the first round, we set $\hat{\mathbf{g}}_0=\mathbf{0}$, which makes CRAFT reduce to the zero-prior minimum-norm initialization. For later rounds, the reference direction preserves temporal consistency, while the correction step guarantees that the final update is adapted to the current participating clients. Thus, the global model maintains a persistent descent direction that is continuously refined by newly sampled clients. This is particularly valuable when the sampled clients in a round are not representative, as the reference can stabilize the global update against local sampling noise.

Meanwhile, unlike momentum methods~\cite{hsu2019measuring,nesterov1983method} that add historical updates directly to the update rule, CRAFT uses history to define a geometric orientation. Rather than simply adding the previous vector, CRAFT searches for the new direction closest to it while satisfying the current conflict-free alignment constraints. See \Cref{app:momentum_comparison} for a detailed comparison between CRAFT and momentum methods.

\subsection{Layer-wise adaptation}
\label{subsec:layerwise_adaptation}
Global conflict resolution treats the entire model parameter vector $\boldsymbol{\theta}$ as a single entity. However, deep neural networks exhibit layer-specific optimization dynamics. Conflicts in shallow layers may signal input distribution shifts, while conflicts in deeper layers may stem from label skew.

To address this, we use a \emph{layer-wise} adaptation. We decompose the problem into $Q$ independent subproblems. For each layer $q \in [Q]$, we construct the normalized update matrix $\mathbf{U}_t^q$ and reference direction $\hat{\mathbf{g}}_t^q$. We then solve the correction problem separately for each layer as
\begin{equation}
\mathbf{g}_t^q
=
\hat{\mathbf{g}}_t^q+\boldsymbol{\Delta}_t^q,
\quad \text{with} \quad
\boldsymbol{\Delta}_t^q
=
\bigl(\mathbf{U}_t^q\bigr)^\dagger
\bigl(\boldsymbol{\rho}_t-\mathbf{U}_t^q\hat{\mathbf{g}}_t^q\bigr).
\label{eq:layerwise_craft_solution}
\end{equation}
This layer-wise approach allows for more flexible conflict resolution, preventing severe conflicts in one layer from suppressing useful components in another. The complete procedure of CRAFT, incorporating this layer-wise strategy, is summarized in Algorithm~\ref{alg:craft}.

\paragraph{Layer grouping.}
The partition depends on the model architecture. In ResNets \cite{he2016deep}, each trainable convolutional or fully connected layer defines a group containing its weights and biases. The trainable parameters of each normalization layer are assigned to the group of the preceding convolutional layer. For example, ResNet-20 consists of one stem convolution, nine residual blocks with two convolutions each, and one fully connected classifier, yielding $Q=20$. Setting $Q=1$ places the entire trainable parameter vector in a single group without layer-wise adaptation.

\begin{algorithm}[t]
\caption{CRAFT framework}
\label{alg:craft}
\begin{algorithmic}[1]
    \STATE {\bfseries Input} $\boldsymbol{\theta}_0$, $T$, $\eta_g$, $\eta_l$, $\{K_i\}_{i=1}^N$, $\{\rho_i\}_{i=1}^N$
    \FOR{$t=0$ {\bfseries to} $T-1$}
    \STATE Server samples client set $\mathcal{S}_t$ and broadcasts $\boldsymbol{\theta}_t$
    \hfill $\triangleright$ \textit{current global model}
    \FOR{each client $i \in \mathcal{S}_t$ \textbf{in parallel}}
    \STATE Client sets $\boldsymbol{\theta}_i^{t,0} \leftarrow \boldsymbol{\theta}_t$
    \FOR{$k=0$ {\bfseries to} $K_i-1$}
    \STATE $\boldsymbol{\theta}_i^{t,k+1} \leftarrow \boldsymbol{\theta}_i^{t,k} - \eta_l \nabla f_i(\boldsymbol{\theta}_i^{t,k};\xi_i^{t,k})$ 
    \hfill $\triangleright$ \textit{local model training, see \eqref{eq:local_sgd}}
    \ENDFOR
    \STATE Client uploads $\mathbf{g}_i^t \leftarrow \boldsymbol{\theta}_t - \boldsymbol{\theta}_i^{t,K_i}$ to the server
    \hfill $\triangleright$ \textit{client update, see \eqref{eq:client_update}}
    \ENDFOR
    
    \STATE Server forms $\boldsymbol{\rho}_t \leftarrow [\,\rho_i / \sum_{j\in\mathcal{S}_t}\rho_j\,]_{i\in\mathcal{S}_t}$
    \hfill $\triangleright$ \textit{active-client weights}
    \FOR{each parameter group $q \in [Q]$}
    \STATE Server constructs $\mathbf{U}_t^q \leftarrow [\,\mathcal{U}(\mathbf{g}_i^{t,q})\,]_{i\in\mathcal{S}_t}^{\top}$
    \hfill $\triangleright$ \textit{normalized local updates}
    \STATE $\hat{\mathbf{g}}_t^q \leftarrow \mathbf{0}$ if $t=0$, else $\mathcal{U}(\mathbf{g}_{t-1}^q)$
    \hfill $\triangleright$ \textit{normalized temporal reference}
    \STATE $\mathbf{g}_t^q \leftarrow \hat{\mathbf{g}}_t^q + \bigl(\mathbf{U}_t^q\bigr)^\dagger \bigl(\boldsymbol{\rho}_t - \mathbf{U}_t^q \hat{\mathbf{g}}_t^q\bigr)$
    \hfill $\triangleright$ \textit{layer-wise projection, see \eqref{eq:layerwise_craft_solution}}
    \ENDFOR
    \STATE Server sets $\mathbf{g}_t \leftarrow \mathrm{concat}_q(\mathbf{g}_t^q)$ 
    \hfill $\triangleright$ \textit{global update concatenation}
    \STATE Server updates $\boldsymbol{\theta}_{t+1} \leftarrow \boldsymbol{\theta}_t - \eta_g \mathbf{g}_t$ 
    \hfill $\triangleright$ \textit{global model update, see \eqref{eq:server_update}}
    \ENDFOR
\end{algorithmic}
\end{algorithm}

\subsection{Communication, memory, and computational cost}
\label{subsec:cost_discussion}

\noindent\textbf{No additional communication.} Each active client uploads the same $d$-dimensional model update as in FedAvg, and all conflict-resolution operations are performed on the server.

\noindent\textbf{Server-side memory cost.} For $m$ active clients, the server stores the current updates in $\mathcal{O}(dm)$ memory and the reference direction in $\mathcal{O}(d)$ memory. The layer-wise projection only adds $\mathcal{O}(m^2)$ reusable temporary workspace for the pseudoinverse computation.

\noindent\textbf{Server-side computational complexity.} For a neural network partitioned into $Q$ parameter groups with dimensions $\{d_q\}_{q=1}^Q$, where $\sum_{q} d_q=d$, forming the intermediate matrices needed for the layer-wise pseudoinverse computations costs $\mathcal{O}(dm^2)$. Solving the $Q$ corresponding pseudoinverse problems costs $\mathcal{O}(Qm^3)$. Thus, the conservative per-round server complexity is $\mathcal{O}(dm^2+Qm^3)$, which is typically dominated by $\mathcal{O}(dm^2)$ when $d\gg m$ in neural networks. Meanwhile, the layer-wise computations are naturally parallelizable.

\subsection{Theoretical analysis}
\label{sec:theory}
We provide a formal analysis of the feasibility of the conflict-free alignment constraints, the common-descent structure, and the finite-time convergence behavior of CRAFT in \Cref{app:convergence}. Here, we briefly summarize the main convergence result. For exposition, the main theorem is stated under full client participation with local gradient-descent training. Let $K_{\max}$ denote the maximum number of local steps across clients. Under this setting, \Cref{thm:full_multi_step_rate} yields
\begin{equation}
\frac1T\sum_{t=0}^{T-1}\|\nabla F(\boldsymbol{\theta}_t)\|^2
=
\mathcal{O}\left(\frac{1}{\sqrt{T}}\right)
+
\mathcal{O}\bigl(\eta_l(K_{\max}-1)\bigr)    
\end{equation}
under the standard choice $\eta_g=\Theta(1/\sqrt{T})$. The first term is the usual non-convex optimization rate, and the second is the deterministic bias induced by multiple local steps, which vanishes when $\eta_l$ is also chosen to decay with $T$. Notably, the analysis does not rely on assumptions about the underlying data distributions and accommodates heterogeneous local update steps $K_i$ across clients.

\section{Experiments}
\label{sec:experiments}
In this section, we evaluate whether CRAFT improves the overall training performance of FL, whether the resulting gains are shared across both head and tail clients, and whether CRAFT can strengthen existing personalized FL methods when used as a direct replacement for weighted averaging.

\subsection{Experimental setup}
\vspace{-2pt}
\label{sec_setups}

\textbf{Datasets, models, and participation.} We use the FEMNIST~\citep{caldas2018leaf} and CIFAR-10/100~\citep{krizhevsky2009learning} datasets. For FEMNIST, we train a multi-layer perceptron (MLP). For CIFAR-10/100, we use convolutional neural networks (CNNs) and residual networks (ResNet-20/56/110) \citep{he2016deep}. FEMNIST uses 1000 clients with 10 active clients per round. For CIFAR, ResNet models use 1000 clients with 100 active clients per round, while CNN models use 100 clients with 10 active clients per round. We simulate non-IID data using highly heterogeneous Dirichlet partitioning \citep{hsu2019measuring} with a concentration parameter $\alpha=0.1$.

\textbf{Baselines.} We compare CRAFT with representative methods covering simple averaging (FedAvg \citep{mcmahan2017communicationefficient}), drift mitigation (FedProx \citep{li2020federated}, FedNova \citep{wang2020fednova}), server momentum/adaptive optimization (FedAvgM \citep{hsu2019measuring}, FedAdam/AdaGrad/Yogi \citep{reddi2021adaptive}), fairness-aware objectives (AFL \citep{mohri2019agnostic}, $q$-FedAvg \citep{li2019fair}, FedMGDA+ \citep{hu2023federatedlearningmeetsmultiobjective}), and conflict-aware aggregation (FedFV \citep{wang2021federatedlearningfairaveraging}, FedLF \citep{pan2024fedlf}, ConFIG \citep{liu2024config}).

Detailed hyperparameter grids and supplementary results are provided in \Cref{app:hyperparameters}.

\subsection{Overall performance}
\label{subsec:overall_performance}

Table~\ref{tab:main_results_1} reports test accuracy metrics for FEMNIST and CIFAR-10, with additional results provided in \Cref{app:hyperparameters}. CRAFT consistently improves mean accuracy, with relative gains of $7.8\%$ over the second-best FedAvgM on FEMNIST and $53.5\%$ over ConFIG on CIFAR-10. Importantly, the improvement is not driven solely by clients that already achieve high accuracy. CRAFT also raises both the best-$10\%$ and worst-$10\%$ client accuracies, indicating simultaneous gains for head and tail users. 

Figure~\ref{fig:convergence} further complements Table~\ref{tab:main_results_1} by plotting convergence curves on more datasets and deeper model architectures. CRAFT's accuracy rises faster and remains higher across these settings. The advantage is especially stable and pronounced on deeper networks, where several baselines either saturate early or improve slowly under client drift. This behavior supports the central design premise of CRAFT, namely that enforcing client-wise alignment during aggregation makes each server update better able to mitigate conflicts, thereby improving both final accuracy and training stability.

\begin{table*}[htbp]
\centering
\caption{Comparative evaluation of test accuracy metrics, including Mean accuracy, Best 10\% (mean accuracy of top-10\% clients), Worst 10\% (mean accuracy of bottom-10\% clients), and Standard Deviation (Std). \textbf{Bold} and \underline{underlined} values indicate the best and second-best results, respectively.}
\label{tab:main_results_1}
\resizebox{\textwidth}{!}{
    \begin{tabular}{lcccccccc}
        \toprule
        & \multicolumn{4}{c}{FEMNIST \& MLP} & \multicolumn{4}{c}{CIFAR-10 \& ResNet-20} \\
        \cmidrule(lr){2-5} \cmidrule(lr){6-9}
        Algorithm & Mean ($\uparrow$) & Best ($\uparrow$) & Worst ($\uparrow$) & Std ($\downarrow$) & Mean ($\uparrow$) & Best ($\uparrow$) & Worst ($\uparrow$) & Std ($\downarrow$) \\
        \midrule
        FedAvg & 0.570 & 0.776 & 0.305 & 0.136 & 0.356 & 0.727 & 0.015 & 0.195 \\
        FedProx & 0.565 & 0.770 & 0.300 & 0.135 & 0.356 & 0.728 & 0.015 & 0.195 \\
        FedNova & 0.589 & 0.804 & 0.337 & 0.135 & 0.381 & 0.734 & 0.040 & \underline{0.190} \\
        FedAvgM & \underline{0.655} & \underline{0.843} & \underline{0.408} & 0.125 & 0.463 & 0.871 & 0.040 & 0.241 \\
        FedAdaGrad & 0.645 & 0.828 & 0.405 & \underline{0.121} & 0.427 & 0.799 & 0.072 & 0.206 \\
        FedAdam & 0.638 & 0.825 & 0.390 & 0.125 & 0.412 & 0.790 & 0.065 & 0.197 \\
        FedYogi & 0.639 & 0.830 & 0.381 & 0.130 & 0.398 & 0.777 & 0.047 & 0.201 \\
        AFL & 0.548 & 0.774 & 0.301 & 0.137 & 0.294 & 0.732 & 0.000 & 0.217 \\
        FedFV & 0.489 & 0.742 & 0.246 & 0.144 & 0.421 & 0.780 & 0.078 & 0.194 \\
        FedMGDA+ & 0.386 & 0.658 & 0.133 & 0.166 & 0.235 & 0.612 & 0.000 & 0.193 \\
        qFedAvg & 0.555 & 0.774 & 0.298 & 0.137 & 0.352 & 0.716 & 0.019 & 0.191 \\
        FedLF & 0.426 & 0.688 & 0.177 & 0.154 & 0.501 & 0.845 & \underline{0.149} & 0.208 \\
        ConFIG & 0.600 & 0.801 & 0.333 & 0.136 & \underline{0.525} & \underline{0.872} & 0.146 & 0.211 \\
        \midrule
        \textbf{CRAFT} & \textbf{0.706} & \textbf{0.878} & \textbf{0.474} & \textbf{0.115} & \textbf{0.806} & \textbf{1.000} & \textbf{0.479} & \textbf{0.159} \\
        \bottomrule
    \end{tabular}
}
\end{table*}

\begin{figure}[ht]
\centering
\includegraphics[width=0.97\linewidth]{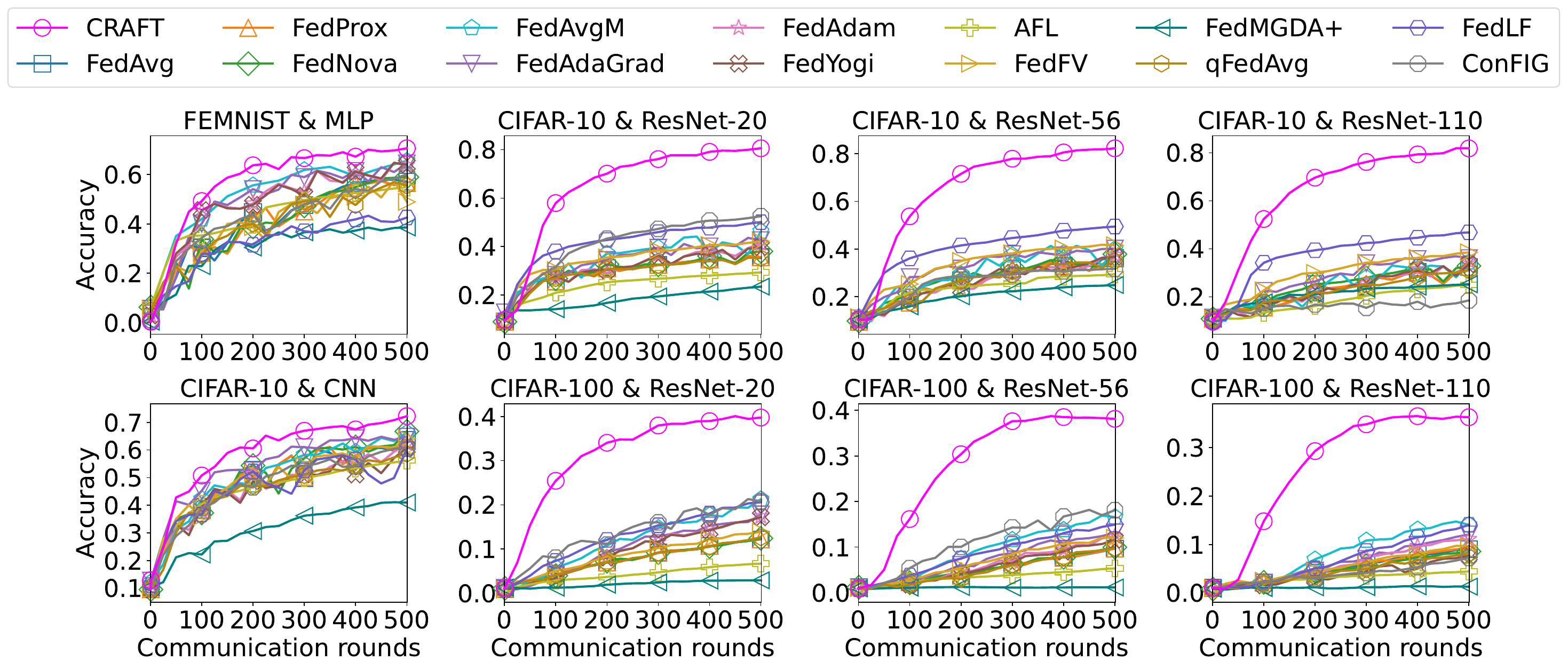}
\caption{Evolution of mean test accuracy over communication rounds across datasets and model architectures. CRAFT consistently converges faster and reaches higher final accuracy, with especially stable gains on more challenging datasets and deeper neural networks.}
\label{fig:convergence}
\end{figure}

\subsection{Per-client accuracy and fairness}
\label{subsec:fairness}

We next evaluate the full distribution of per-client accuracy. Mean accuracy alone can hide severe disparity because two global models with the same mean may behave very differently across clients.

The results in Table~\ref{tab:main_results_1} show that CRAFT improves mean accuracy without sacrificing either end of the client distribution. On FEMNIST, relative to the second-best FedAvgM, CRAFT improves best-10\% and worst-10\% accuracy by $4.2\%$ and $16.2\%$, respectively. On CIFAR-10 with ResNet-20, relative to the corresponding second-best methods, CRAFT improves best-10\% accuracy by $14.7\%$ over ConFIG and worst-10\% accuracy by $221.5\%$ over FedLF. These gains indicate that CRAFT does not simply trade tail users for head users or vice versa. 

Figure~\ref{fig:client_accuracy_distribution} makes this distributional change explicit by plotting the final test accuracy of all clients under FedAvg and CRAFT\@. The x-axis is test accuracy and the y-axis counts clients, so a distribution shifted to the right means that more clients achieve higher accuracy. On FEMNIST and CIFAR-10, CRAFT moves most of the mass to the right and reduces the low-accuracy tail. On CIFAR-100, the FedAvg distribution can appear concentrated, but this concentration is misleading, as most clients remain near zero accuracy. CRAFT moves the distribution away from this near-zero regime substantially.

\begin{figure}[ht]
\centering
\includegraphics[width=0.9\linewidth]{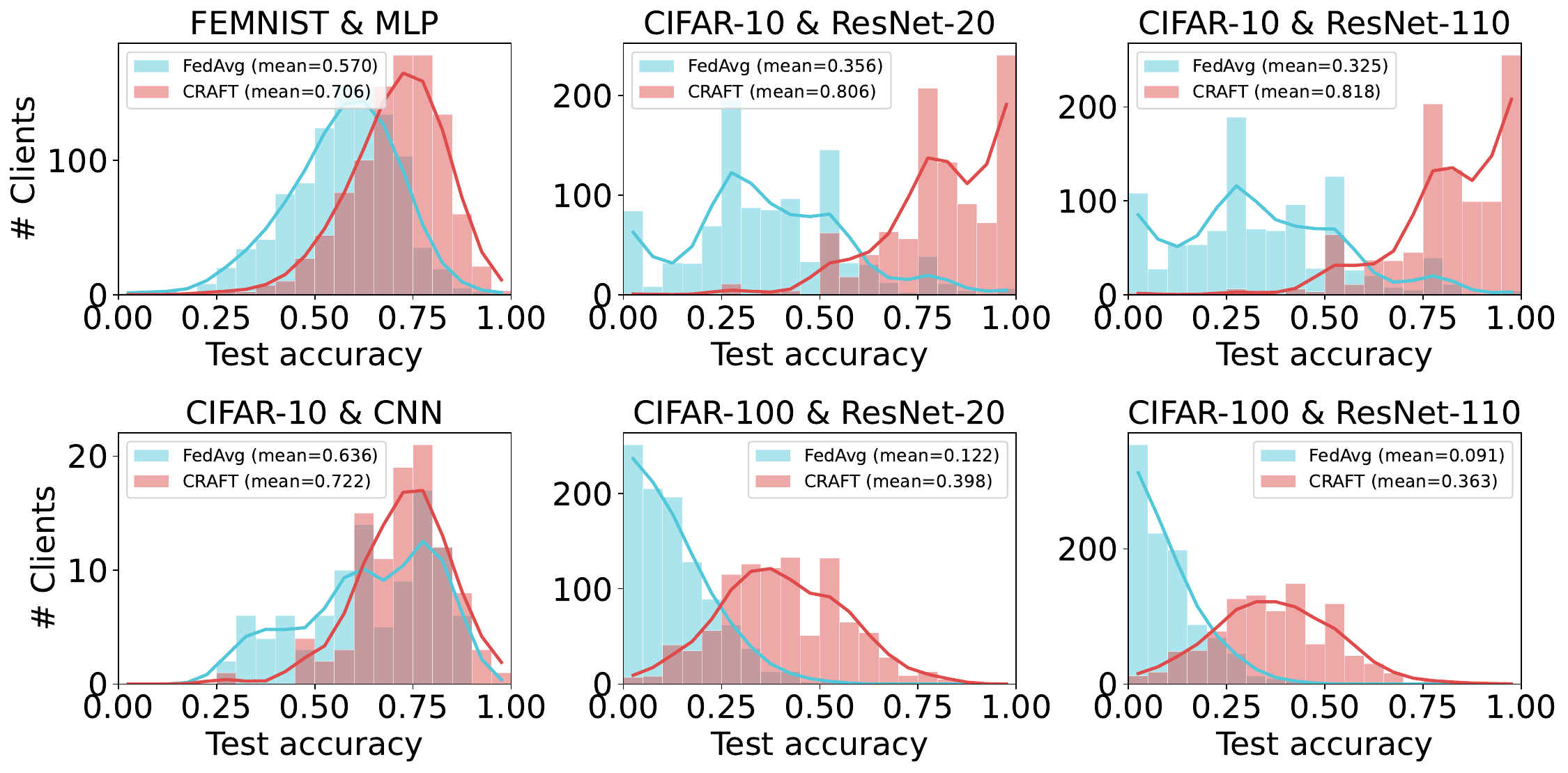}
\caption{Distribution of per-client test accuracy. A rightward shift indicates higher accuracy for more clients, and CRAFT produces a more favorable distribution with improved mean and tail performance.}
\label{fig:client_accuracy_distribution}
\end{figure}

\begin{figure}[ht]
\centering
\includegraphics[width=0.95\linewidth]{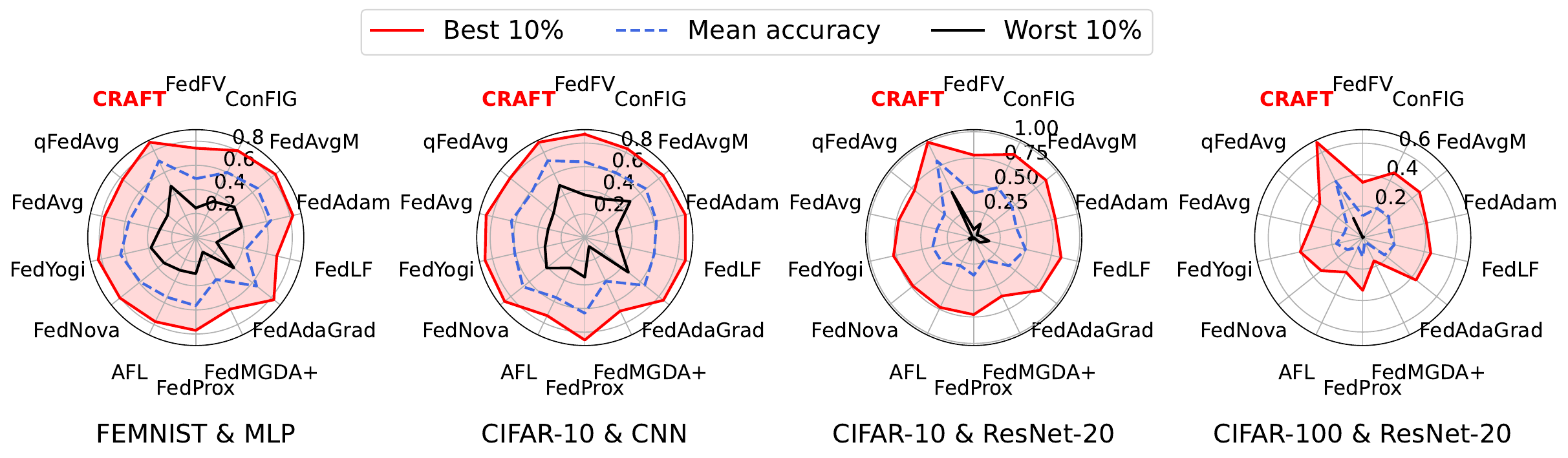}
\caption{Radar plot of accuracy metrics, with each spoke corresponding to a method. CRAFT improves the worst-client performance while keeping the mean and best performance high.}
\label{fig:client_radar}
\end{figure}

Moreover, Figure~\ref{fig:client_radar} provides a complementary summary of mean, best-10\%, and worst-10\% accuracy. In this radar plot, each spoke corresponds to an algorithm, and the three curves correspond to best-10\%, mean, and worst-10\% accuracy, all of which are better when farther from the center. Several baselines obtain reasonable mean accuracy while showing a large gap between the best and worst clients, which reflects a polarized client distribution. CRAFT keeps the mean and best-client performance high while lifting the worst-client curve, showing a much fairer distribution.

\subsection{Integration with personalized federated learning}
\label{subsec:personalized_integration}

CRAFT provides a new aggregation operator rather than a personalized FL algorithm that maintains a separate model for each client. Nevertheless, personalized FL methods such as Ditto still require an aggregation step to update a shared model or backbone. We therefore test whether CRAFT can improve a representative personalized FL pipeline by replacing the averaging step inside Ditto~\citep[Alg.\@ 2, line 8]{li2021ditto} with the CRAFT operator. This keeps Ditto's client-side personalization mechanism unchanged and modifies only the aggregation rule used for the shared update.

Figure~\ref{fig:app_ditto_craft} reports heatmaps over different server and client learning-rate combinations. Compared with vanilla Ditto, Ditto with CRAFT improves mean test accuracy across the grid. The improvement is especially clear at larger learning rates. These results show that CRAFT has broad plug-in extensibility. It can strengthen personalized FL methods by replacing their internal aggregation step.

\begin{figure}[ht]
\centering
\begin{subfigure}{0.49\linewidth}
    \centering
    \includegraphics[width=0.95\linewidth]{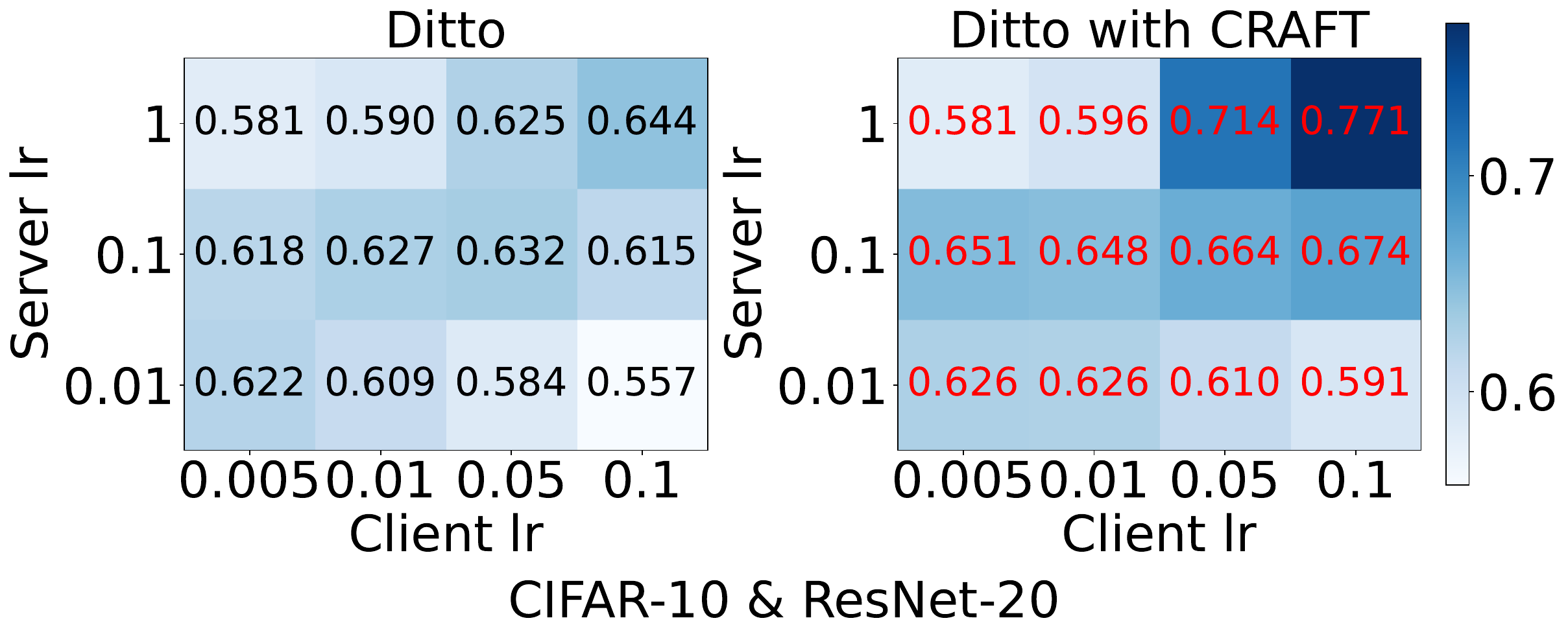}
    \label{fig:app_ditto_craft_left}
\end{subfigure}
\begin{subfigure}{0.49\linewidth}
    \centering
    \includegraphics[width=0.95\linewidth]{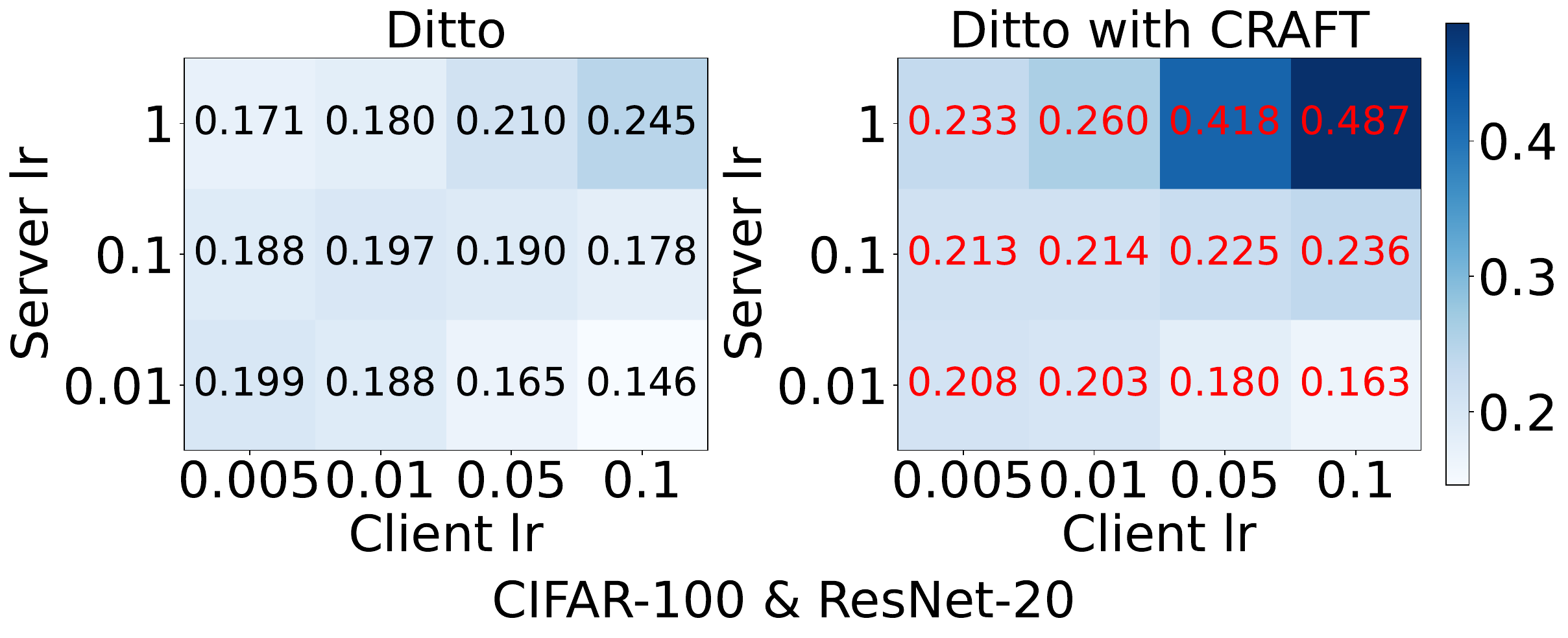}
    \label{fig:app_ditto_craft_right}
\end{subfigure}
\caption{Comparison between vanilla Ditto and Ditto with CRAFT across various settings. Replacing Ditto's average aggregation with the CRAFT operator yields average relative improvements of $6.0\%$ on CIFAR-10 and $32.0\%$ on CIFAR-100, demonstrating CRAFT's strong extensibility.}
\label{fig:app_ditto_craft}
\end{figure}

\subsection{Ablation studies on CRAFT components}
\label{subsec:ablation_studies}

\paragraph{Layer-wise adaptation and momentum-like reference.}
We compare four configurations to examine the contributions of layer-wise adaptation and the reference direction. CRAFT uses the momentum-like reference with layer-wise adaptation, while CRAFT-no-Layerwise retains the same reference without layer-wise adaptation. CRAFT-no-Momentum (ConFIG-Layerwise) uses a zero reference with layer-wise adaptation, and ConFIG uses a zero reference without layer-wise adaptation.

We use CIFAR-10 and CIFAR-100 with ResNet-20 and $N=1000$ clients. Table~\ref{tab:component_participation_ablation} fixes the Dirichlet parameter at $\alpha=0.1$ and varies the number of active clients per round, $m\in\{10,50,100\}$, reporting test accuracy at round 500. Additional comparisons with baselines across heterogeneity levels at fixed $m=100$ are provided in \Cref{app:heterogeneity_robustness}.

\begin{table}[htbp]
\centering
\caption{Test accuracy of different CRAFT component configurations with varying active-client count $m$ at fixed Dirichlet $\alpha=0.1$ on CIFAR with ResNet-20 and $N=1000$ clients.}
\label{tab:component_participation_ablation}
\begingroup
\small
\setlength{\tabcolsep}{5pt}
\begin{tabular}{lcccccc}
    \toprule
    & \multicolumn{3}{c}{CIFAR-10} & \multicolumn{3}{c}{CIFAR-100} \\
    \cmidrule(lr){2-4} \cmidrule(lr){5-7}
    Algorithm & $m=10$ & $m=50$ & $m=100$ & $m=10$ & $m=50$ & $m=100$ \\
    \midrule
    FedAvg & 0.260 & 0.337 & 0.356 & 0.092 & 0.117 & 0.122 \\
    ConFIG & 0.300 & 0.545 & 0.525 & 0.090 & 0.167 & 0.208 \\
    CRAFT-no-Layerwise & \underline{0.361} & \underline{0.613} & 0.462 & \underline{0.133} & 0.189 & 0.152 \\
    CRAFT-no-Momentum & 0.297 & 0.609 & \underline{0.698} & 0.098 & \underline{0.206} & \underline{0.292} \\
    \midrule
    \textbf{CRAFT} & \textbf{0.417} & \textbf{0.731} & \textbf{0.806} & \textbf{0.188} & \textbf{0.391} & \textbf{0.398} \\
    \bottomrule
\end{tabular}
\endgroup
\end{table}

As shown in Table~\ref{tab:component_participation_ablation}, CRAFT achieves the highest accuracy in every setting and consistently outperforms CRAFT-no-Layerwise, with a larger gap as participation increases. The full CRAFT method also outperforms CRAFT-no-Momentum throughout, supporting the effectiveness of combining layer-wise adaptation with the momentum-like reference. With a zero reference, layer-wise adaptation improves upon ConFIG in five of the six settings.

\paragraph{Sensitivity to layer-wise grouping.}
We further compare different layer partitions $Q\in\{1,2,5,11,20\}$. The grouping definitions and results across heterogeneity levels are provided in \Cref{app:layerwise_grouping}.
As summarised in Table~\ref{tab:layerwise_grouping}, either $Q=11$ or $Q=20$ achieves the highest accuracy in all six participation settings. Both consistently outperform the broad $Q=1$ and $Q=2$ partitions, supporting conflict resolution at residual-block or layer granularity. The milder-heterogeneity results in \Cref{app:layerwise_grouping} favor $Q=11$, showing that accuracy does not improve monotonically with $Q$. These results support the layer-level default as an effective grouping rule without data-dependent tuning, while indicating that the best granularity can depend on the use case.

\begin{table}[htbp]
\centering
\caption{Test accuracy of different layer-wise grouping strategies with varying active-client count $m$ at fixed Dirichlet $\alpha=0.1$ on CIFAR with ResNet-20 and $N=1000$ clients.}
\label{tab:layerwise_grouping}
\begingroup
\small
\setlength{\tabcolsep}{5pt}
\begin{tabular}{lcccccc}
    \toprule
    & \multicolumn{3}{c}{CIFAR-10} & \multicolumn{3}{c}{CIFAR-100} \\
    \cmidrule(lr){2-4} \cmidrule(lr){5-7}
    Configuration & $m=10$ & $m=50$ & $m=100$ & $m=10$ & $m=50$ & $m=100$ \\
    \midrule
    CRAFT ($Q=1$) & 0.361 & 0.613 & 0.462 & 0.133 & 0.189 & 0.152 \\
    CRAFT ($Q=2$) & 0.350 & 0.643 & 0.530 & 0.156 & 0.212 & 0.312 \\
    CRAFT ($Q=5$) & 0.334 & 0.698 & 0.754 & 0.180 & 0.322 & 0.394 \\
    CRAFT ($Q=11$) & \underline{0.364} & \underline{0.730} & \underline{0.798} & \underline{0.187} & \underline{0.368} & \textbf{0.415} \\
    \midrule
    \textbf{CRAFT ($Q=20$)} & \textbf{0.417} & \textbf{0.731} & \textbf{0.806} & \textbf{0.188} & \textbf{0.391} & \underline{0.398} \\
    \bottomrule
\end{tabular}

\endgroup
\end{table}

\FloatBarrier
\section{Conclusions and limitations}
\label{sec:limitations}
This work presented CRAFT, a conflict-resolved aggregation framework that treats federated aggregation as a geometric correction problem. By enforcing positive alignment between the global update and each participating client update, CRAFT mitigates conflicts in heterogeneous FL\@. Our analysis identifies the resulting common-descent structure and establishes convergence guarantees. Experiments show that CRAFT improves mean accuracy while maintaining strong performance for both head and tail clients. CRAFT also transfers directly to personalized FL as a drop-in aggregation module, improving personalization without modifying local adaptation. Overall, CRAFT provides a new perspective on aggregation design and a practical tool for improving FL under heterogeneity.

Meanwhile, several limitations remain. First, CRAFT uses the previous global update as the reference direction, which is simple and empirically effective. Alternatives such as moving averages over multiple previous rounds or adaptive references based on recent client conflicts may provide smoother or more informative guidance. Second, the alignment target is currently chosen according to clients' data proportions. More adaptive targets could reweight clients based on training progress, loss, gradient conflicts, or fairness requirements. Understanding how these reference and alignment choices affect accuracy, stability, and client-level fairness is an important direction for future work.

\clearpage
\section*{Acknowledgments}
Ziqi Wang acknowledges funding from the European Union's Horizon Europe MSCA project ModConFlex (grant number 101073558). Qiang Liu acknowledges support from the China Scholarship Council (No. 202206120036). We are grateful to Enrique Zuazua for invaluable discussions. We acknowledge the scientific support and HPC resources provided by the Erlangen National High Performance Computing Center (NHR@FAU) of the Friedrich-Alexander-Universität Erlangen-Nürnberg (FAU). The hardware is partially funded by the German Research Foundation (DFG).

\bibliography{references}
\bibliographystyle{plainnat}

\newpage
\appendix
\crefalias{section}{appendix}
\crefalias{subsection}{subappendix}
\crefalias{subsubsection}{subsubappendix}
\onecolumn

{\Large\bfseries Appendix}

This appendix provides details about CRAFT on the following topics:
\begin{itemize}
\item \Cref{app:convergence}: Technical details of the theoretical analysis.
\item \Cref{app:additional_discussion}: Additional discussion and design choices.
\item \Cref{app:hyperparameters}: Training details and hyperparameter settings.
\item \Cref{app:additional_results}: Additional numerical results and visualizations.
\end{itemize}

\section{Technical details of the theoretical analysis}
\label{app:convergence}
In this appendix, we first discuss the feasibility of the conflict-free alignment constraints, then prove the full-vector convergence bound, and finally analyze the layer-wise version and residual alignment constraints as extensions.

\subsection{Feasibility and solution existence}
\label{app:feasibility_discussion}

It follows from \eqref{eq:craft_opt} that CRAFT's update satisfies the linear system $\mathbf{U}_t \mathbf{g} = \boldsymbol{\rho}_t$ with $\mathbf{U}_t\in\mathbb{R}^{m\times d}$. The system is feasible if and only if $\boldsymbol{\rho}_t\in\operatorname{Range}(\mathbf{U}_t)$, equivalently $\operatorname{rank}(\mathbf{U}_t)= \operatorname{rank}([\mathbf{U}_t\mid\boldsymbol{\rho}_t])$. In the exact statements in this appendix, $\mathbf{U}_t$ denotes the unit-normalized matrix with rows $\mathbf{u}_i^t=\mathbf{g}_i^t/\|\mathbf{g}_i^t\|$ for nonzero uploaded updates. Stabilized normalization with $\varepsilon>0$ and the handling of near-zero updates are treated as numerical variants in \Cref{app:numerical_active_set}.

In deep neural networks, the parameter dimension $d$ is typically orders of magnitude larger than the number of sampled clients $m$. Moreover, under non-IID data distributions, the uploaded local directions arise from distinct local optimization landscapes. An exact linear dependence among the $m$ active-client directions in $\mathbb{R}^d$ is therefore a degenerate event. Thus, in the generic high-dimensional case, $\operatorname{rank}(\mathbf{U}_t)=m$. Then $\operatorname{Range}(\mathbf{U}_t)=\mathbb{R}^m$, so every target $\boldsymbol{\rho}_t\in\mathbb{R}^m$ is attainable. Since $d \gg m$, the linear system is underdetermined and its feasible set contains infinitely many solutions. CRAFT resolves this non-uniqueness by choosing the solution closest to the momentum-like reference $\hat{\mathbf{g}}_t$, which is the Euclidean projection onto the affine set $\{\mathbf{g}\in\mathbb{R}^d \mid \mathbf{U}_t\mathbf{g}=\boldsymbol{\rho}_t\}$ and is given by \eqref{eq:craft_solution}.

If the active directions are rank-deficient, exact feasibility is no longer automatic. When $\boldsymbol{\rho}_t\notin\operatorname{Range}(\mathbf{U}_t)$, the Moore--Penrose update in \eqref{eq:craft_solution} should be interpreted as a least-squares projection in the correction variable, \emph{i.e.}, $\mathbf{g}_t$ is the point closest to $\hat{\mathbf{g}}_t$ among all updates that minimize $\|\mathbf{U}_t\mathbf{g}-\boldsymbol{\rho}_t\|_2^2$. The resulting residual, denoted by $\boldsymbol{\delta}_t$, is analyzed in Corollary~\ref{cor:approx_proj_extension}.

\subsection{Convergence analysis for CRAFT}
\label{app:full_multi_step}

\subsubsection{Assumptions}
\begin{assumption}[$L$-Smoothness]
\label{ass:full_smooth}
Each local loss $f_i$ is $L$-smooth, \emph{i.e.}, for all $\boldsymbol{\theta},\boldsymbol{\theta}' \in \mathbb{R}^d$,
\begin{equation*}
    \|\nabla f_i(\boldsymbol{\theta}')
    -
    \nabla f_i(\boldsymbol{\theta})\|
    \le
    L\|\boldsymbol{\theta}'-\boldsymbol{\theta}\|.
\end{equation*}
Hence, $F(\boldsymbol{\theta})=\sum_{i=1}^N\rho_i f_i(\boldsymbol{\theta})$ is also $L$-smooth.
\end{assumption}

\begin{assumption}[Lower-bounded objective]
\label{ass:full_lower_bound}
The global objective is bounded below, \emph{i.e.}, there exists a constant $F^\star \in \mathbb{R}$ such that
$F(\boldsymbol{\theta}) \geq F^\star$ for all $\boldsymbol{\theta} \in \mathbb{R}^d$.
\end{assumption}

\begin{assumption}[Bounded gradients and updates]
\label{ass:full_bounded_quantities}
There exist constants $G_f,G>0$ such that $\|\nabla f_i(\boldsymbol{\theta})\|\le G_f$ for all $i$ and all $\boldsymbol{\theta}$, and the update direction satisfies $\|\mathbf{g}_t\|\le G$ for every round $t$.
\end{assumption}

These are common regularity conditions in non-convex FL analysis. The lower-boundedness assumption defines $F^\star$ as a finite lower bound rather than as an attained global minimizer, and it is used only when telescoping objective values. Bounded gradient assumptions also appear in FL analyses such as \citep{Li2020On,reddi2021adaptive, yu2019parallel}. The bounded-update condition can be enforced in practice through norm clipping controls. If such clipping is applied after the projection step, however, it may change the equality constraint and is therefore covered by the projection residual in Corollary~\ref{cor:approx_proj_extension}. Importantly, our analysis does not impose any distributional assumption on the client datasets and allows heterogeneous local training steps $K_i$ for each client $i\in[N]$.

\subsubsection{Main results}
We first define the constants that appear in the main bound.
\begin{equation}
\label{eq:B_drift_def}
K_{\max}:=\max_{i\in[N]}K_i,
\quad 
\rho_{\min}:=\min_{i\in[N]}\rho_i,
\quad
c:=\frac{\rho_{\min}}{G_f},
\quad
B_{\mathrm{drift}}
:=
\frac{L G_f}{2}
\left(
G+\sum_{i=1}^N \rho_i^2
\right).
\end{equation}

\begin{theorem}[Convergence Bound for CRAFT]
\label{thm:full_multi_step_rate}
Consider CRAFT under full client participation and local gradient-descent training in the feasible full-vector alignment-constrained setting. Suppose Assumptions~\ref{ass:full_smooth}--\ref{ass:full_bounded_quantities} hold. Then, for any $T\ge1$,
\begin{equation}
    \label{eq:full_multi_step_main_bound}
    \frac{1}{T}
    \sum_{t=0}^{T-1}
    \|\nabla F(\boldsymbol{\theta}_t)\|^2
    \le
    \frac{F(\boldsymbol{\theta}_0)-F^\star}{c\eta_g T}
    +
    \frac{L\eta_g G^2}{2c}
    +
    \frac{B_{\mathrm{drift}}}{c}\eta_l(K_{\max}-1).
\end{equation}
\end{theorem}

Before proving the theorem, we state the standard rate obtained from the chosen step sizes.

\begin{corollary}[CRAFT Rate under Step-Size Schedules]
\label{cor:full_multi_step_rate}
Under the conditions of \Cref{thm:full_multi_step_rate}, if the server learning rate is chosen as $\eta_g=\Theta(1/\sqrt{T})$, then
\begin{equation}
    \frac{1}{T}
    \sum_{t=0}^{T-1}
    \|\nabla F(\boldsymbol{\theta}_t)\|^2
    =
    \mathcal{O}\left(\frac{1}{\sqrt{T}}\right)
    +
    \mathcal{O}\bigl(\eta_l(K_{\max}-1)\bigr).
\end{equation}
If, in addition, $K_{\max}$ is fixed and the local learning rate is scheduled as $\eta_l=\mathcal{O}(1/\sqrt{T})$, then the deterministic drift term also decays, and we recover the standard non-convex convergence rate:
\begin{equation}
    \frac{1}{T}
    \sum_{t=0}^{T-1}
    \|\nabla F(\boldsymbol{\theta}_t)\|^2
    =
    \mathcal{O}\left(\frac{1}{\sqrt{T}}\right).
\end{equation}
\end{corollary}

\paragraph{Sketch of the proof.} The proof of \Cref{thm:full_multi_step_rate} is deferred until after stating three supporting lemmas. The conflict-free alignment constraint enforces positive alignment between $\mathbf{g}_t$ and each accumulated client update $\mathbf{g}_i^t$, while multi-step local training makes $\mathbf{g}_i^t$ drift from the idealized gradient $\eta_lK_i\nabla f_i(\boldsymbol{\theta}_t)$. Lemma~\ref{lem:multi_step_decomp} quantifies this local drift, Lemma~\ref{lem:weighted_grad_lower} converts the resulting client-wise lower bounds into a global-gradient lower bound, and Lemma~\ref{lem:multi_step_alignment} combines them into a common-descent certificate. The theorem then follows from the smoothness descent lemma and telescoping over communication rounds.

In Lemma~\ref{lem:multi_step_decomp}, we first quantify the drift error incurred by replacing the sum of gradients along the local trajectory with the gradient evaluated at the round-start broadcast model.

\begin{lemma}[Local Drift Bound]
\label{lem:multi_step_decomp}
Under Assumptions~\ref{ass:full_smooth} and~\ref{ass:full_bounded_quantities}, for each client $i$ and round $t$,
\begin{equation}
    \label{eq:multi_step_decomp}
    \mathbf{g}_i^t
    =
    \eta_lK_i\nabla f_i(\boldsymbol{\theta}_t)
    +
    \mathbf{e}_i^t,
\end{equation}
with
\begin{equation}
    \label{eq:multi_step_drift_bound}
    \|\mathbf{e}_i^t\|
    \le
    \frac{L G_f}{2}\eta_l^2K_i(K_i-1).
\end{equation}
\end{lemma}

\begin{proof}
Adding and subtracting $\nabla f_i(\boldsymbol{\theta}_t)$ in the local client updates gives
\begin{equation}
    \mathbf{g}_i^t
    =
    \eta_l
    \sum_{k=0}^{K_i-1}
    \nabla f_i(\boldsymbol{\theta}_i^{t,k})
    =
    \eta_l K_i\nabla f_i(\boldsymbol{\theta}_t)
    +
    \mathbf{e}_i^t,
\end{equation}
where
\begin{equation}
    \mathbf{e}_i^t
    :=
    \eta_l
    \sum_{k=0}^{K_i-1}
    \left(
    \nabla f_i(\boldsymbol{\theta}_i^{t,k})
    -
    \nabla f_i(\boldsymbol{\theta}_t)
    \right).
\end{equation}
For $k=0$, the difference is zero. For $k\ge1$, the local recursion gives
\begin{equation}
    \boldsymbol{\theta}_i^{t,k}-\boldsymbol{\theta}_t
    =
    -\eta_l
    \sum_{r=0}^{k-1}
    \nabla f_i(\boldsymbol{\theta}_i^{t,r}).
\end{equation}
Taking norms and applying the triangle inequality gives
\begin{equation}\begin{aligned}
    \|\boldsymbol{\theta}_i^{t,k}-\boldsymbol{\theta}_t\|
    &=
    \left\|
    -\eta_l
    \sum_{r=0}^{k-1}
    \nabla f_i(\boldsymbol{\theta}_i^{t,r})
    \right\| \\
    &\le
    \eta_l
    \sum_{r=0}^{k-1}
    \|\nabla f_i(\boldsymbol{\theta}_i^{t,r})\| \\
    &\le
    \eta_l kG_f.
\end{aligned}\end{equation}
We now bound the accumulated drift error.
\begin{equation}
\begin{aligned}
    \|\mathbf{e}_i^t\|
    &=
    \left\|
    \eta_l
    \sum_{k=0}^{K_i-1}
    \left(
    \nabla f_i(\boldsymbol{\theta}_i^{t,k})
    -
    \nabla f_i(\boldsymbol{\theta}_t)
    \right)
    \right\| \\
    &\le
    \eta_l
    \sum_{k=0}^{K_i-1}
    \left\|
    \nabla f_i(\boldsymbol{\theta}_i^{t,k})
    -
    \nabla f_i(\boldsymbol{\theta}_t)
    \right\|  \\
    &\le
    \eta_l
    \sum_{k=0}^{K_i-1}
    L\|\boldsymbol{\theta}_i^{t,k}-\boldsymbol{\theta}_t\|  \\
    &\le
    L G_f \eta_l^2
    \sum_{k=0}^{K_i-1} k\\
    &=
    \frac{L G_f}{2}\eta_l^2K_i(K_i-1).
\end{aligned}
\end{equation}

\end{proof}

\begin{lemma}[Weighted Gradient Lower Bound]
\label{lem:weighted_grad_lower}
Under Assumption~\ref{ass:full_bounded_quantities}, for any round $t$, let
\begin{equation}
    \nabla f_i^t:=\nabla f_i(\boldsymbol{\theta}_t),
    \qquad
    \nabla F_t:=\nabla F(\boldsymbol{\theta}_t).
\end{equation}
Then
\begin{equation}
    \label{eq:weighted_grad_lower}
    \sum_{i=1}^N\rho_i^2\|\nabla f_i^t\|
    \ge
    \rho_{\min}\|\nabla F_t\|
    \ge
    \frac{\rho_{\min}}{G_f}\|\nabla F_t\|^2.
\end{equation}
\end{lemma}

\begin{proof}
The first step uses only the definition of $\rho_{\min}$ and the probability weights.
\begin{equation}\begin{aligned}
    \sum_{i=1}^N\rho_i^2\|\nabla f_i^t\|
    &=
    \sum_{i=1}^N
    \rho_i
    \bigl(\rho_i\|\nabla f_i^t\|\bigr) \\
    &\ge
    \rho_{\min}
    \sum_{i=1}^N\rho_i\|\nabla f_i^t\| \\
    &\ge
    \rho_{\min}
    \left\|
    \sum_{i=1}^N\rho_i\nabla f_i^t
    \right\| \\
    &=
    \rho_{\min}
    \|\nabla F_t\|.
\end{aligned}\end{equation}
For the squared-gradient conversion, bounded local gradients imply
\begin{equation}\begin{aligned}
    \|\nabla F_t\|
    &=
    \left\|
    \sum_{i=1}^N\rho_i\nabla f_i^t
    \right\| \\
    &\le
    \sum_{i=1}^N\rho_i\|\nabla f_i^t\|
    \le
    G_f,
\end{aligned}\end{equation}
and therefore
\begin{equation}
    \rho_{\min}\|\nabla F_t\|
    \ge
    \frac{\rho_{\min}}{G_f}\|\nabla F_t\|^2.
\end{equation}
Combining the inequalities proves the claim.
\end{proof}

\begin{lemma}[Common Descent Certificate]
\label{lem:multi_step_alignment}
In the feasible full-vector alignment-constrained setting described above, under Assumptions~\ref{ass:full_smooth} and~\ref{ass:full_bounded_quantities},
\begin{equation}
    \label{eq:multi_step_alignment}
    \langle \nabla F(\boldsymbol{\theta}_t),\mathbf{g}_t\rangle
    \ge
    c\|\nabla F(\boldsymbol{\theta}_t)\|^2
    -
    B_{\mathrm{drift}}\,\eta_l(K_{\max}-1),
\end{equation}
where $c$ and $B_{\mathrm{drift}}$ are defined in \eqref{eq:B_drift_def}.
\end{lemma}

\begin{proof}
Let $\nabla f_i^t:=\nabla f_i(\boldsymbol{\theta}_t)$ and $\nabla F_t:=\nabla F(\boldsymbol{\theta}_t)$. The exact conflict-free alignment constraint and the definition $\mathbf{u}_i^t=\mathbf{g}_i^t/\|\mathbf{g}_i^t\|$ imply
\begin{equation}\begin{aligned}
    \langle \mathbf{g}_i^t,\mathbf{g}_t\rangle
    &=
    \|\mathbf{g}_i^t\|
    \langle \mathbf{u}_i^t,\mathbf{g}_t\rangle \\
    &=
    \rho_i\|\mathbf{g}_i^t\|.
\end{aligned}\end{equation}
Using the drift decomposition \eqref{eq:multi_step_decomp}, each client gradient has the following alignment with the projected direction.
\begin{equation}\begin{aligned}
    \eta_lK_i
    \langle \nabla f_i^t,\mathbf{g}_t\rangle
    &\overset{\eqref{eq:multi_step_decomp}}{=}
    \langle \mathbf{g}_i^t-\mathbf{e}_i^t,\mathbf{g}_t\rangle  \\
    &=
    \langle \mathbf{g}_i^t,\mathbf{g}_t\rangle
    -
    \langle \mathbf{e}_i^t,\mathbf{g}_t\rangle \\
    &=
    \rho_i\|\mathbf{g}_i^t\|
    -
    \langle \mathbf{e}_i^t,\mathbf{g}_t\rangle \\
    &\ge
    \rho_i\|\mathbf{g}_i^t\|
    -
    \|\mathbf{e}_i^t\|\,\|\mathbf{g}_t\|.
\end{aligned}\end{equation}
The uploaded displacement norm is lower bounded by the reverse triangle inequality.
\begin{equation}\begin{aligned}
    \|\mathbf{g}_i^t\|
    &\overset{\eqref{eq:multi_step_decomp}}{=}
    \|\eta_lK_i\nabla f_i^t+\mathbf{e}_i^t\|
    \ge
    \eta_lK_i\|\nabla f_i^t\|
    -
    \|\mathbf{e}_i^t\|.
\end{aligned}\end{equation}
Substituting this lower bound and dividing by $\eta_lK_i$ gives
\begin{equation}\begin{aligned}
    \langle \nabla f_i^t,\mathbf{g}_t\rangle
    &\ge
    \rho_i\|\nabla f_i^t\|
    -
    \frac{\rho_i+\|\mathbf{g}_t\|}{\eta_lK_i}
    \|\mathbf{e}_i^t\| \\
    &\overset{\substack{\eqref{eq:multi_step_drift_bound}}}{\ge}
    \rho_i\|\nabla f_i^t\|
    -
    \frac{L G_f}{2}
    (\rho_i+G)\eta_l(K_i-1),
\end{aligned}\end{equation}
where the division is valid because $K_i\ge1$ and $\eta_l>0$. Multiplying by $\rho_i$ and summing over clients yields
\begin{equation}\begin{aligned}
    \langle \nabla F_t,\mathbf{g}_t\rangle
    &=
    \sum_{i=1}^N
    \rho_i
    \langle \nabla f_i^t,\mathbf{g}_t\rangle \\
    &\ge
    \sum_{i=1}^N
    \rho_i^2\|\nabla f_i^t\|
    -
    \frac{L G_f}{2}\eta_l
    \sum_{i=1}^N
    \rho_i(\rho_i+G)(K_i-1) \\
    &\overset{\substack{\eqref{eq:B_drift_def}}}{\ge}
    \sum_{i=1}^N
    \rho_i^2\|\nabla f_i^t\|
    -
    B_{\mathrm{drift}}\eta_l(K_{\max}-1) \\
    &\ge
    c\|\nabla F_t\|^2
    -
    B_{\mathrm{drift}}\eta_l(K_{\max}-1).
\end{aligned}\end{equation}
The last inequality uses Lemma~\ref{lem:weighted_grad_lower}, which gives \eqref{eq:multi_step_alignment}.
\end{proof}

With the common-descent bound in place, we can prove the main convergence result.
\begin{proof}[Proof of \Cref{thm:full_multi_step_rate}]
Applying $L$-smoothness of $F$ to $\boldsymbol{\theta}_{t+1}=\boldsymbol{\theta}_t-\eta_g\mathbf{g}_t$, we have
\begin{equation}\begin{aligned}
    F(\boldsymbol{\theta}_{t+1})
    &\le
    F(\boldsymbol{\theta}_t)
    -
    \eta_g
    \langle \nabla F(\boldsymbol{\theta}_t),\mathbf{g}_t\rangle
    +
    \frac{L\eta_g^2}{2}\|\mathbf{g}_t\|^2 \\
    &\le
    F(\boldsymbol{\theta}_t)
    -
    c\eta_g\|\nabla F(\boldsymbol{\theta}_t)\|^2
    +
    \eta_gB_{\mathrm{drift}}\eta_l(K_{\max}-1)
    +
    \frac{L\eta_g^2}{2}G^2.
\end{aligned}\end{equation}
The second inequality above uses Lemma~\ref{lem:multi_step_alignment}. Rearranging this one-step descent inequality gives
\begin{equation}
    c\eta_g\|\nabla F(\boldsymbol{\theta}_t)\|^2
    \le
    F(\boldsymbol{\theta}_t)-F(\boldsymbol{\theta}_{t+1})
    +
    \eta_gB_{\mathrm{drift}}\eta_l(K_{\max}-1)
    +
    \frac{L\eta_g^2}{2}G^2.
\end{equation}
Summing from $t=0$ to $T-1$, the objective differences telescope.
\begin{equation}\begin{aligned}
    c\eta_g
    \sum_{t=0}^{T-1}
    \|\nabla F(\boldsymbol{\theta}_t)\|^2
    &\le
    \sum_{t=0}^{T-1}
    \bigl[
    F(\boldsymbol{\theta}_t)-F(\boldsymbol{\theta}_{t+1})
    \bigr]
    +
    T\eta_gB_{\mathrm{drift}}\eta_l(K_{\max}-1)
    +
    \frac{L\eta_g^2}{2}TG^2 \\
    &=
    F(\boldsymbol{\theta}_0)-F(\boldsymbol{\theta}_T)
    +
    T\eta_gB_{\mathrm{drift}}\eta_l(K_{\max}-1)
    +
    \frac{L\eta_g^2}{2}TG^2 \\
    &\le
    F(\boldsymbol{\theta}_0)-F^\star
    +
    T\eta_gB_{\mathrm{drift}}\eta_l(K_{\max}-1)
    +
    \frac{L\eta_g^2}{2}TG^2.
\end{aligned}\end{equation}
Dividing both sides by $c\eta_gT$ gives \eqref{eq:full_multi_step_main_bound}.
\end{proof}

\paragraph{Interpretation of the bound.} When $K_{\max}=1$, the drift term vanishes and the theorem reduces to the one-step bound. For $K_{\max}>1$, local computation introduces a stationarity neighborhood of size $\mathcal{O}(\eta_l(K_{\max}-1))$. By scheduling $\eta_l$ to decay with $T$, this deterministic bias can be forced to zero.

\subsubsection{Average gradient-norm bound}
\label{app:average_grad_norm_bound}

The squared-gradient result above uses the bounded-gradient assumption to convert a linear descent certificate into a squared stationarity measure. If the desired stationarity measure is instead the average gradient norm, the same argument yields a direct bound with alignment constant $\rho_{\min}$.

\begin{lemma}[Linear Common-Descent Bound]
\label{lem:linear_multi_step_alignment}
In the feasible full-vector alignment-constrained setting described above, under Assumptions~\ref{ass:full_smooth} and~\ref{ass:full_bounded_quantities},
\begin{equation}
    \label{eq:linear_multi_step_alignment}
    \langle \nabla F(\boldsymbol{\theta}_t),\mathbf{g}_t\rangle
    \ge
    \rho_{\min}\|\nabla F(\boldsymbol{\theta}_t)\|
    -
    B_{\mathrm{drift}}\,\eta_l(K_{\max}-1),
\end{equation}
where $B_{\mathrm{drift}}$ is defined in \eqref{eq:B_drift_def}.
\end{lemma}

\begin{proof}
The proof of Lemma~\ref{lem:multi_step_alignment} first establishes the pre-conversion bound
\begin{equation}\begin{aligned}
    \langle \nabla F_t,\mathbf{g}_t\rangle
    &\ge
    \sum_{i=1}^N
    \rho_i^2\|\nabla f_i^t\|
    -
    B_{\mathrm{drift}}\eta_l(K_{\max}-1) \\
    &\ge
    \rho_{\min}\|\nabla F_t\|
    -
    B_{\mathrm{drift}}\eta_l(K_{\max}-1).
\end{aligned}\end{equation}
The second inequality uses Lemma~\ref{lem:weighted_grad_lower}. Equivalently, this uses only
\begin{equation}
    \sum_{i=1}^N\rho_i^2\|\nabla f_i^t\|
    \ge
    \rho_{\min}\|\nabla F_t\|,
\end{equation}
and does not convert $\|\nabla F_t\|$ into $\|\nabla F_t\|^2/G_f$. This proves \eqref{eq:linear_multi_step_alignment}.
\end{proof}

\begin{corollary}[Average Gradient-Norm Bound]
\label{cor:average_grad_norm_bound}
Under the conditions of \Cref{thm:full_multi_step_rate}, for any $T\ge1$,
\begin{equation}
    \label{eq:average_grad_norm_bound}
    \frac{1}{T}
    \sum_{t=0}^{T-1}
    \|\nabla F(\boldsymbol{\theta}_t)\|
    \le
    \frac{F(\boldsymbol{\theta}_0)-F^\star}{\rho_{\min}\eta_g T}
    +
    \frac{L\eta_g G^2}{2\rho_{\min}}
    +
    \frac{B_{\mathrm{drift}}}{\rho_{\min}}\eta_l(K_{\max}-1).
\end{equation}
Consequently, choosing $\eta_g=\Theta(1/\sqrt{T})$ gives
\begin{equation}
    \frac{1}{T}
    \sum_{t=0}^{T-1}
    \|\nabla F(\boldsymbol{\theta}_t)\|
    =
    \mathcal{O}\left(\frac{1}{\sqrt{T}}\right)
    +
    \mathcal{O}\bigl(\eta_l(K_{\max}-1)\bigr).
\end{equation}
If $K_{\max}$ is fixed and $\eta_l=\mathcal{O}(1/\sqrt{T})$, then the drift term also decays and
\begin{equation}
    \frac{1}{T}
    \sum_{t=0}^{T-1}
    \|\nabla F(\boldsymbol{\theta}_t)\|
    =
    \mathcal{O}\left(\frac{1}{\sqrt{T}}\right).
\end{equation}
\end{corollary}

\begin{proof}
By $L$-smoothness of $F$ and the update rule $\boldsymbol{\theta}_{t+1}=\boldsymbol{\theta}_t-\eta_g\mathbf{g}_t$, we have
\begin{equation}\begin{aligned}
    F(\boldsymbol{\theta}_{t+1})
    &\le
    F(\boldsymbol{\theta}_t)
    -
    \eta_g
    \langle \nabla F(\boldsymbol{\theta}_t),\mathbf{g}_t\rangle
    +
    \frac{L\eta_g^2}{2}\|\mathbf{g}_t\|^2 \\
    &\le
    F(\boldsymbol{\theta}_t)
    -
    \rho_{\min}\eta_g
    \|\nabla F(\boldsymbol{\theta}_t)\|
    +
    \eta_gB_{\mathrm{drift}}\eta_l(K_{\max}-1)
    +
    \frac{L\eta_g^2}{2}G^2.
\end{aligned}\end{equation}
The second inequality above uses Lemma~\ref{lem:linear_multi_step_alignment}. Rearranging, summing, and telescoping gives
\begin{equation}\begin{aligned}
    \rho_{\min}\eta_g
    \sum_{t=0}^{T-1}
    \|\nabla F(\boldsymbol{\theta}_t)\|
    &\le
    \sum_{t=0}^{T-1}
    \bigl[
    F(\boldsymbol{\theta}_t)-F(\boldsymbol{\theta}_{t+1})
    \bigr]
    +
    T\eta_gB_{\mathrm{drift}}\eta_l(K_{\max}-1)
    +
    \frac{L\eta_g^2}{2}TG^2 \\
    &=
    F(\boldsymbol{\theta}_0)-F(\boldsymbol{\theta}_T)
    +
    T\eta_gB_{\mathrm{drift}}\eta_l(K_{\max}-1)
    +
    \frac{L\eta_g^2}{2}TG^2 \\
    &\le
    F(\boldsymbol{\theta}_0)-F^\star
    +
    T\eta_gB_{\mathrm{drift}}\eta_l(K_{\max}-1)
    +
    \frac{L\eta_g^2}{2}TG^2.
\end{aligned}\end{equation}
Dividing by $\rho_{\min}\eta_gT$ gives \eqref{eq:average_grad_norm_bound}. The step-size statements follow by substitution.
\end{proof}

\subsection{Layer-wise multi-step projection extension}
\label{app:layerwise_multi_step}
The layer-wise analysis follows the full-vector argument block by block. The only change is that the drift constant and update-norm bound are accumulated over layers. We use
\begin{equation}
\label{eq:layerwise_Bq}
B_q
:=
\frac{L_qG_f}{2}
\left(
G^q
+
\sum_{i=1}^N\rho_i^2
\right),
\qquad
B_{\mathrm{layer}}
:=
\sum_{q=1}^Q B_q,
\qquad
G_{\mathrm{layer}}^2
:=
\sum_{q=1}^Q
\bigl(G^q\bigr)^2.
\end{equation}

\begin{corollary}[Layer-Wise Multi-Step Projection]
\label{cor:layerwise_multi_step_rate}
Consider the full-participation layer-wise CRAFT update \eqref{eq:layerwise_craft_solution} with $Q$ layers, indexed by $q$. Under Assumptions~\ref{ass:full_smooth}--\ref{ass:full_bounded_quantities}, suppose that, for every client $i$ and layer $q$, $\nabla^q f_i$ is $L_q$-Lipschitz with respect to the full parameter norm, and that $\|\mathbf{g}_t^q\|\le G^q$ for all $t,q$. Let $c=\rho_{\min}/G_f$ as in \eqref{eq:B_drift_def}. Then, for any $T\ge1$,
\begin{equation}
    \label{eq:layerwise_multi_main_bound}
    \frac{1}{T}
    \sum_{t=0}^{T-1}
    \|\nabla F(\boldsymbol{\theta}_t)\|^2
    \le
    \frac{F(\boldsymbol{\theta}_0)-F^\star}{c\eta_g T}
    +
    \frac{L\eta_g G_{\mathrm{layer}}^2}{2c}
    +
    \frac{B_{\mathrm{layer}}}{c}\eta_l(K_{\max}-1).
\end{equation}
Consequently, choosing $\eta_g=\Theta(1/\sqrt{T})$ gives the same asymptotic form as Corollary~\ref{cor:full_multi_step_rate}. If $K_{\max}$ is fixed and $\eta_l=\mathcal{O}(1/\sqrt{T})$, the deterministic drift term also decays.
\end{corollary}

\begin{proof}
Fix a layer $q$. The proof follows the full-vector argument applied to the coordinates of this block, with $L_q$ replacing the global smoothness constant in the drift term. Define
\begin{equation}
    \mathbf{e}_i^{t,q}
    :=
    \eta_l\sum_{k=0}^{K_i-1}
    \left(
    \nabla^q f_i(\boldsymbol{\theta}_i^{t,k})
    -
    \nabla^q f_i(\boldsymbol{\theta}_t)
    \right).
\end{equation}
As in Lemma~\ref{lem:multi_step_decomp}, we have
\begin{equation}
    \|\boldsymbol{\theta}_i^{t,k}-\boldsymbol{\theta}_t\|
    \le
    \eta_l kG_f
\end{equation}
and the block Lipschitz condition gives
\begin{equation}\begin{aligned}
    \mathbf{g}_i^{t,q}
    &=
    \eta_lK_i\nabla^q f_i(\boldsymbol{\theta}_t)
    +
    \mathbf{e}_i^{t,q},
    \\
    \|\mathbf{e}_i^{t,q}\|
    &\le
    \frac{L_qG_f}{2}\eta_l^2K_i(K_i-1).
\end{aligned}\end{equation}
The exact layer-wise constraint gives, for every client $i$,
\begin{equation}
    \left\langle
    \mathbf{g}_i^{t,q},
    \mathbf{g}_t^q
    \right\rangle
    =
    \rho_i\|\mathbf{g}_i^{t,q}\|.
\end{equation}
Therefore,
\begin{equation}\begin{aligned}
    \eta_lK_i
    \left\langle
    \nabla^q f_i(\boldsymbol{\theta}_t),
    \mathbf{g}_t^q
    \right\rangle
    &=
    \left\langle
    \mathbf{g}_i^{t,q}-\mathbf{e}_i^{t,q},
    \mathbf{g}_t^q
    \right\rangle \\
    &\ge
    \rho_i\|\mathbf{g}_i^{t,q}\|
    -
    \|\mathbf{e}_i^{t,q}\|\,\|\mathbf{g}_t^q\|.
\end{aligned}\end{equation}
The reverse triangle inequality gives
\begin{equation}
    \|\mathbf{g}_i^{t,q}\|
    \ge
    \eta_lK_i\|\nabla^q f_i(\boldsymbol{\theta}_t)\|
    -
    \|\mathbf{e}_i^{t,q}\|.
\end{equation}
Dividing by $\eta_lK_i$, using $\|\mathbf{g}_t^q\|\le G^q$, and substituting the block drift bound gives
\begin{equation}
    \left\langle
    \nabla^q f_i(\boldsymbol{\theta}_t),
    \mathbf{g}_t^q
    \right\rangle
    \ge
    \rho_i\|\nabla^q f_i(\boldsymbol{\theta}_t)\|
    -
    \frac{L_qG_f}{2}(\rho_i+G^q)\eta_l(K_i-1).
\end{equation}
Multiplying by $\rho_i$, summing over clients, using $K_i-1\le K_{\max}-1$, and using $\sum_i\rho_i=1$ yields
\begin{equation}\begin{aligned}
    \left\langle
    \nabla^qF(\boldsymbol{\theta}_t),
    \mathbf{g}_t^q
    \right\rangle
    &\ge
    \sum_{i=1}^N
    \rho_i^2
    \|\nabla^q f_i(\boldsymbol{\theta}_t)\|
    -
    B_q\eta_l(K_{\max}-1) \\
    &\ge
    c
    \left\|
    \nabla^qF(\boldsymbol{\theta}_t)
    \right\|^2
    -
    B_q\eta_l(K_{\max}-1).
\end{aligned}\end{equation}
The second inequality is the block version of Lemma~\ref{lem:weighted_grad_lower}. It is valid because $\|\nabla^qF(\boldsymbol{\theta}_t)\|\le \|\nabla F(\boldsymbol{\theta}_t)\|\le G_f$. Summing over layers gives
\begin{equation}\begin{aligned}
    \left\langle
    \nabla F(\boldsymbol{\theta}_t),
    \mathbf{g}_t
    \right\rangle
    &=
    \sum_{q=1}^Q
    \left\langle
    \nabla^qF(\boldsymbol{\theta}_t),
    \mathbf{g}_t^q
    \right\rangle \\
    &\ge
    c
    \sum_{q=1}^Q
    \left\|
    \nabla^qF(\boldsymbol{\theta}_t)
    \right\|^2
    -
    B_{\mathrm{layer}}\eta_l(K_{\max}-1) \\
    &=
    c\|\nabla F(\boldsymbol{\theta}_t)\|^2
    -
    B_{\mathrm{layer}}\eta_l(K_{\max}-1).
\end{aligned}\end{equation}
The layer-wise update norm satisfies
\begin{equation}\begin{aligned}
    \|\mathbf{g}_t\|^2
    &=
    \sum_{q=1}^Q
    \|\mathbf{g}_t^q\|^2
    \le
    G_{\mathrm{layer}}^2.
\end{aligned}\end{equation}
Applying $L$-smoothness of $F$ to the update $\boldsymbol{\theta}_{t+1} = \boldsymbol{\theta}_t-\eta_g\mathbf{g}_t$, we obtain
\begin{equation}\begin{aligned}
    F(\boldsymbol{\theta}_{t+1})
    &\le
    F(\boldsymbol{\theta}_t)
    -
    \eta_g
    \left\langle
    \nabla F(\boldsymbol{\theta}_t),
    \mathbf{g}_t
    \right\rangle
    +
    \frac{L\eta_g^2}{2}
    \|\mathbf{g}_t\|^2 \\
    &\le
    F(\boldsymbol{\theta}_t)
    -
    c\eta_g
    \|\nabla F(\boldsymbol{\theta}_t)\|^2
    +
    \eta_gB_{\mathrm{layer}}\eta_l(K_{\max}-1)
    +
    \frac{L\eta_g^2}{2}G_{\mathrm{layer}}^2.
\end{aligned}\end{equation}
Rearranging and summing over $t=0,\ldots,T-1$ yields
\begin{equation}\begin{aligned}
    c\eta_g
    \sum_{t=0}^{T-1}
    \|\nabla F(\boldsymbol{\theta}_t)\|^2
    &\le
    \sum_{t=0}^{T-1}
    \bigl[
    F(\boldsymbol{\theta}_t)-F(\boldsymbol{\theta}_{t+1})
    \bigr]
    +
    T\eta_gB_{\mathrm{layer}}\eta_l(K_{\max}-1)
    +
    \frac{L\eta_g^2}{2}TG_{\mathrm{layer}}^2 \\
    &\le
    F(\boldsymbol{\theta}_0)-F^\star
    +
    T\eta_gB_{\mathrm{layer}}\eta_l(K_{\max}-1)
    +
    \frac{L\eta_g^2}{2}TG_{\mathrm{layer}}^2.
\end{aligned}\end{equation}
Dividing by $c\eta_gT$ gives \eqref{eq:layerwise_multi_main_bound}.
\end{proof}

\paragraph{Interpretation.} Layer-wise projection is therefore a direct block-wise generalization of \Cref{thm:full_multi_step_rate}. The convergence rate keeps the same asymptotic form, while the drift constant becomes a sum of layer-specific terms. When some block-gradient smoothness constants $L_q$ are smaller than the global constant $L$, the resulting $B_q$ terms give a sharper drift bound.

\subsection{Approximate alignment under residual constraints}
\label{sec:approx_proj_extension}

The preceding results use the exact conflict-free alignment constraint $\mathbf{U}_t\mathbf{g}_t=\boldsymbol{\rho}$. In implementation, the update may satisfy this relation only approximately. This can happen when the linear system is solved in a least-squares sense, when the exact constraint is infeasible, or when the final update is modified by practical post-processing such as clipping. The next corollary isolates the effect of such deviations through an explicit residual term. It requires only that the unit-normalized alignment relation hold up to a residual vector. When the projection is exact, this residual is zero and the bound reduces to \Cref{thm:full_multi_step_rate}. Throughout the statement, the uploaded client updates used to form the unit directions are assumed to be nonzero.

\begin{corollary}[Approximate Alignment under Residual Constraints]
\label{cor:approx_proj_extension}
Consider the full-vector, full-participation setting of \Cref{thm:full_multi_step_rate} under Assumptions~\ref{ass:full_smooth}--\ref{ass:full_bounded_quantities}, and suppose all uploaded updates are nonzero. Let $\mathbf{U}_t$ denote the unit-normalized matrix. Suppose that, instead of the exact relation $\mathbf{U}_t\mathbf{g}_t=\boldsymbol{\rho}$, the computed update satisfies
\begin{equation}
    \mathbf{U}_t\mathbf{g}_t
    =
    \boldsymbol{\rho}
    +
    \boldsymbol{\delta}_t.
\end{equation}
Define
\begin{equation}
    R_t
    :=
    \left(
    G_f+\frac{L G_f}{2}\eta_l(K_{\max}-1)
    \right)
    \sum_{i=1}^N\rho_i|\delta_{t,i}|.
\end{equation}
Then, for any $T\ge1$,
\begin{equation}
    \label{eq:approx_proj_main_bound}
    \frac{1}{T}
    \sum_{t=0}^{T-1}
    \|\nabla F(\boldsymbol{\theta}_t)\|^2
    \le
    \frac{F(\boldsymbol{\theta}_0)-F^\star}{c\eta_g T}
    +
    \frac{L\eta_g G^2}{2c}
    +
    \frac{B_{\mathrm{drift}}}{c}\eta_l(K_{\max}-1)
    +
    \frac{1}{cT}\sum_{t=0}^{T-1}R_t.
\end{equation}
\end{corollary}

\begin{proof}
The residual relation gives the perturbed client-alignment identity
\begin{equation}\begin{aligned}
    \langle \mathbf{g}_i^t,\mathbf{g}_t\rangle
    &=
    \|\mathbf{g}_i^t\|
    \langle \mathbf{u}_i^t,\mathbf{g}_t\rangle \\
    &=
    (\rho_i+\delta_{t,i})\|\mathbf{g}_i^t\|
    \\
    &\ge
    \rho_i\|\mathbf{g}_i^t\|
    -
    |\delta_{t,i}|\|\mathbf{g}_i^t\|.
\end{aligned}\end{equation}
Relative to the exact proof, the only new term is the residual contribution. For each client,
\begin{equation}\begin{aligned}
    \frac{\|\mathbf{g}_i^t\|}{\eta_lK_i}
    &\overset{\eqref{eq:multi_step_decomp}}{\le}
    \|\nabla f_i^t\|
    +
    \frac{\|\mathbf{e}_i^t\|}{\eta_lK_i} \\
    &\overset{\eqref{eq:multi_step_drift_bound}}{\le}
    G_f
    +
    \frac{L G_f}{2}\eta_l(K_i-1) \\
    &\le
    G_f
    +
    \frac{L G_f}{2}\eta_l(K_{\max}-1).
\end{aligned}\end{equation}
Combining this upper bound with the same Cauchy--Schwarz and reverse-triangle steps used in Lemma~\ref{lem:multi_step_alignment}, the client-wise lower bound becomes
\begin{equation}\begin{aligned}
    \langle \nabla f_i^t,\mathbf{g}_t\rangle
    \ge&
    \rho_i\|\nabla f_i^t\|
    -
    \frac{L G_f}{2}(\rho_i+G)\eta_l(K_i-1) \\
    &-
    |\delta_{t,i}|
    \left(
    G_f+\frac{L G_f}{2}\eta_l(K_{\max}-1)
    \right),
\end{aligned}\end{equation}
where the first two terms are exactly the drift terms from Lemma~\ref{lem:multi_step_alignment}. Multiplying by $\rho_i$, summing over clients, applying Lemma~\ref{lem:weighted_grad_lower}, and using the definition of $R_t$ gives
\begin{equation}\begin{aligned}
    \langle \nabla F(\boldsymbol{\theta}_t),\mathbf{g}_t\rangle
    &\ge
    c\|\nabla F(\boldsymbol{\theta}_t)\|^2
    -
    B_{\mathrm{drift}}\eta_l(K_{\max}-1)
    -
    R_t.
\end{aligned}\end{equation}
Substituting this inequality into the $L$-smooth descent bound gives
\begin{equation}\begin{aligned}
    F(\boldsymbol{\theta}_{t+1})
    &\le
    F(\boldsymbol{\theta}_t)
    -
    \eta_g
    \langle \nabla F(\boldsymbol{\theta}_t),\mathbf{g}_t\rangle
    +
    \frac{L\eta_g^2}{2}\|\mathbf{g}_t\|^2 \\
    &\le
    F(\boldsymbol{\theta}_t)
    -
    c\eta_g\|\nabla F(\boldsymbol{\theta}_t)\|^2
    +
    \eta_gB_{\mathrm{drift}}\eta_l(K_{\max}-1)
    +
    \eta_gR_t
    +
    \frac{L\eta_g^2}{2}G^2.
\end{aligned}\end{equation}
Rearranging, summing, using $F(\boldsymbol{\theta}_T)\ge F^\star$, and dividing by $c\eta_gT$ gives \eqref{eq:approx_proj_main_bound}.
\end{proof}

\newpage
\section{Additional discussion and design choices}
\label{app:additional_discussion}

\subsection{Comparison with momentum-based and adaptive server optimizers}
\label{app:momentum_comparison}

CRAFT is not a momentum variant or a naive combination of ConFIG~\citep{liu2024config} and momentum methods~\cite{nesterov1983method,polyak1964some,sutskever2013importance}. Momentum-based server optimizers smooth a raw aggregate, for example
\begin{equation}
\mathbf{v}_t
=
\beta \mathbf{v}_{t-1}
+ (1-\beta)\mathbf{a}_t,
\end{equation}
where $\mathbf{a}_t$ is the aggregated current-round update. By contrast, CRAFT solves the reference-anchored projection problem
\begin{equation}
\mathbf{g}_t
=
\arg\min_{\mathbf{g}}
\frac{1}{2}\|\mathbf{g}-\hat{\mathbf{g}}_t\|^2
\quad
\textnormal{s.t.}
\quad
\mathbf{U}_t\mathbf{g}=\boldsymbol{\rho}_t,
\end{equation}
where $\hat{\mathbf{g}}_t$ is the reference direction constructed from the previous global update. The resulting update is the projection of this reference onto the affine set defined by the current conflict-free alignment constraints, not additive smoothing of the averaged update.

This distinction matters because CRAFT depends on the geometry of the active client updates. Two active-client sets can have the same averaged update but different pairwise angles or inner products. Momentum and adaptive server optimizers that operate on the average would produce the same raw aggregate in such cases, whereas CRAFT also enforces the client-wise alignment constraints and can therefore produce different updates. Thus, our contribution reformulates aggregation from zero-prior reconstruction, as in ConFIG, or temporal smoothing, as in FedAvgM methods, into a reference-anchored constrained projection that explicitly avoids conflicts with each active client.

\subsection{Choice of the alignment target \texorpdfstring{$\boldsymbol{\rho}_t$}{rho\_t}}
\label{app:rho_choice}

In our implementation, the alignment target $\boldsymbol{\rho}_t$ in \eqref{eq:linear_constraints}--\eqref{eq:craft_solution} is computed from the dataset sizes of the active clients. At round $t$, for each $i\in\mathcal{S}_t$, we set
\begin{equation}
\rho_i^t
=
\frac{|\mathcal{D}_i|}
{\sum_{j\in\mathcal{S}_t}|\mathcal{D}_j|}.
\end{equation}
The vector $\boldsymbol{\rho}_t=[\rho_i^t]_{i\in\mathcal{S}_t}$ is ordered consistently with the rows of $\mathbf{U}_t$. This choice matches the FL objective \eqref{eq:global_obj}, where clients are weighted by their data proportions. Under partial participation, FedAvg uses the same renormalized dataset-size weights when averaging the updates from the sampled clients. Thus, $\boldsymbol{\rho}_t$ is not an extra heuristic parameter. It is the round-wise counterpart of the data-proportional weights used in federated optimization.

At the same time, $\boldsymbol{\rho}_t$ plays a structural role in CRAFT\@. It specifies the target alignment levels in $\mathbf{U}_t\mathbf{g}_t=\boldsymbol{\rho}_t$. A larger $\rho_i^t$ requires stronger alignment with client $i$'s normalized update direction and therefore gives that client greater influence in the aggregated update. In this sense, CRAFT keeps the same relative weighting principle as FedAvg, but changes the aggregation geometry so that the final update remains positively aligned with every active client.

More generally, different choices of $\boldsymbol{\rho}_t$ encode different aggregation priorities. The default data-size-based target is natural when the goal is consistency with the global FL objective. A uniform $\boldsymbol{\rho}_t$, with $\rho_i^t=1/m$ for each of the $m$ active clients, treats active clients equally regardless of local sample size. \Cref{app:alignment_target_results} compares this equal target with the data-proportional target. Adaptive targets based on client reliability, heterogeneity, training progress, or fairness requirements are also possible and remain an interesting direction for future work.

\subsection{Numerical active set for near-zero updates}
\label{app:numerical_active_set}

The equality constraint $\mathbf{U}_t\mathbf{g}_t=\boldsymbol{\rho}_t$ is intended to enforce positive alignment with informative clients. It should not assign a fixed positive target to a vanishing update. Indeed, for a stabilized row
\begin{equation}
\widetilde{\mathbf{u}}_i^t
=
\frac{\mathbf{g}_i^t}{\|\mathbf{g}_i^t\|+\varepsilon},
\end{equation}
the constraint $\langle \widetilde{\mathbf{u}}_i^t,\mathbf{g}_t\rangle=\rho_i^t>0$ implies
\begin{equation}
\|\mathbf{g}_t\|
\ge
\rho_i^t\frac{\|\mathbf{g}_i^t\|+\varepsilon}{\|\mathbf{g}_i^t\|}.
\end{equation}
Thus, if $\|\mathbf{g}_i^t\|$ is close to zero while the target $\rho_i^t>0$ and the stabilizer $\varepsilon>0$ remain fixed, the factor $\varepsilon/\|\mathbf{g}_i^t\|$ can make the constraint numerically ill conditioned. For this reason, a more stable numerical implementation uses an active set for projection. In the layer-wise version, for each layer $q$, we define
\begin{equation}
\mathcal A_t^q
=
\{i\in\mathcal S_t \mid \|\mathbf{g}_i^{t,q}\|\ge\tau_q\},
\end{equation}
where $\tau_q>0$ is a numerical tolerance, \emph{e.g.}, $\tau_q = 10^{-6}$. The alignment matrix is then built as
\begin{equation}
\mathbf{U}_{t,\mathcal A}^q
=
\bigl[\mathcal{U}(\mathbf{g}_i^{t,q})\bigr]_{i\in\mathcal A_t^q}^{\top}.
\end{equation}
Equivalently, the alignment target is magnitude-gated.
\begin{equation}
\bar\rho_i^{t,q}
=
\begin{cases}
    \displaystyle
    \frac{\rho_i^t}{\sum_{j\in\mathcal A_t^q}\rho_j^t},
    & i\in\mathcal A_t^q,\\
    0,
    & i\notin\mathcal A_t^q,
\end{cases}
\end{equation}
whenever $\mathcal A_t^q\neq\emptyset$. Here $\bar{\boldsymbol{\rho}}_{t,\mathcal A}^q :=[\bar\rho_i^{t,q}]_{i\in\mathcal A_t^q}$ denotes the vector of active entries, indexed in the same order as the rows of $\mathbf{U}_{t,\mathcal A}^q$. CRAFT then solves
\begin{equation}
\mathbf{g}_t^q
=
\hat{\mathbf{g}}_t^q
+
\bigl(\mathbf{U}_{t,\mathcal A}^q\bigr)^\dagger
\bigl(
\bar{\boldsymbol{\rho}}_{t,\mathcal A}^q
-
\mathbf{U}_{t,\mathcal A}^q\hat{\mathbf{g}}_t^q
\bigr).
\end{equation}
If $\mathcal A_t^q=\emptyset$, the projection for that layer is skipped and we use the ordinary weighted average update for that layer, which is already small under the same threshold criterion.

Note that the above active-set rule is not a client-sampling change. At each round, every selected client still performs local training and uploads its update. The rule only decides whether an update is sufficiently large to define a reliable row of the conflict-free alignment constraint. A near-zero update is therefore assigned a zero alignment target for numerical stability, and it automatically re-enters the conflict-free alignment set once its norm exceeds the tolerance.

\subsection{Initialization and early-stage robustness}
\label{app:initialization_robustness}

CRAFT is not sensitive to the initial direction because the reference is used as a soft prior, not as a hard constraint. At $t=0$, we set the initial reference to zero. The first update therefore reduces to the minimum-norm feasible solution with a zero reference, namely the ConFIG-style update. This provides a neutral initialization and avoids injecting arbitrary directional bias at the start of training.

For $t>0$, the previous update serves only as a reference orientation. At each round, it is reprojected onto the conflict-free feasible set defined by the newly sampled active clients. Consequently, even if the early active-client samples are skewed, the induced bias cannot be carried forward unchanged. Any component inherited from the previous reference that is inconsistent with the current client geometry is removed by construction through the projection. CRAFT therefore has an intrinsic self-correction mechanism that mitigates early-stage sampling bias and prevents persistent error propagation across rounds.

\FloatBarrier
\clearpage
\section{Hyperparameter settings}
\label{app:hyperparameters}

This appendix provides the detailed experimental setup, including dataset descriptions, model architectures, scalability settings, baselines, and hyperparameter choices used in Section~\ref{sec:experiments}.

\textbf{Datasets.} We evaluate on three widely used benchmarks, FEMNIST~\citep{caldas2018leaf}, CIFAR-10, and CIFAR-100~\citep{krizhevsky2009learning}. FEMNIST is a naturally non-IID dataset containing handwritten characters from 62 classes. For CIFAR-10 and CIFAR-100, we construct heterogeneous client datasets using Dirichlet partitions \citep{hsu2019measuring} with a concentration parameter $\alpha=0.1$, where smaller $\alpha$ indicates stronger heterogeneity. We manually set the minimum number of samples per client to 20. For each experiment, every client locally splits its assigned data into 80\% training and 20\% test samples. We also provide an illustration of the heterogeneous data distribution in Figure~\ref{fig:data_distribution}.

\textbf{Models.} For FEMNIST, we use a multi-layer perceptron (MLP) with two hidden layers of 200 units each. For CIFAR-10/100, we evaluate both convolutional neural networks (CNNs) and residual networks (ResNets \cite{he2016deep}). The CNN consists of two convolutional layers followed by three fully connected layers, as in \cite{pan2024fedlf}. For ResNets, we use ResNet-20/56/110. Following the setup of \cite{reddi2021adaptive}, we replace batch normalization layers with group normalization layers using a group size of one.

\textbf{Hyperparameters.}
\label{app:hyperparameter_selection}
For all experiments, we perform a grid search of learning rates and report the best results. We search over client learning rates $\eta_l \in \{0.005, 0.01, 0.05, 0.1\}$ and server learning rates $\eta_g \in \{0.01, 0.1, 1\}$. For client training, we apply a per-round learning rate decay of $0.999$. For methods with specific hyperparameters, such as $q$ for q-FedAvg, we use the default values recommended in their original papers. We set the batch size to 50 for MLP and CNN experiments and 20 for ResNet experiments. Each client performs one local epoch per communication round.

\textbf{Baselines.} We compare CRAFT with representative methods covering simple averaging (FedAvg \citep{mcmahan2017communicationefficient}), drift mitigation (FedProx \citep{li2020federated}, FedNova \citep{wang2020fednova}), server momentum/adaptive optimization (FedAvgM \citep{hsu2019measuring}, FedAdam/AdaGrad/Yogi \citep{reddi2021adaptive}), fairness-aware objectives (AFL \citep{mohri2019agnostic}, $q$-FedAvg \citep{li2019fair}, FedMGDA+ \citep{hu2023federatedlearningmeetsmultiobjective}), and conflict-aware aggregation (FedFV \citep{wang2021federatedlearningfairaveraging}, FedLF \citep{pan2024fedlf}, ConFIG \citep{liu2024config}). We omit the history-based adaptation, as explained in \Cref{app_fedlf_history}.

\textbf{Scalability.} To assess scalability, we use a large-scale setting with $N=1000$ clients for FEMNIST and CIFAR experiments with ResNet models. To also evaluate medium-scale cross-silo settings, we conduct CIFAR experiments with CNNs using $N=100$ clients. At each global round, we sample 10 clients for MLP and CNN experiments, and 100 clients for ResNet experiments.

\begin{figure}[ht]
\centering
\includegraphics[width=0.7\linewidth]{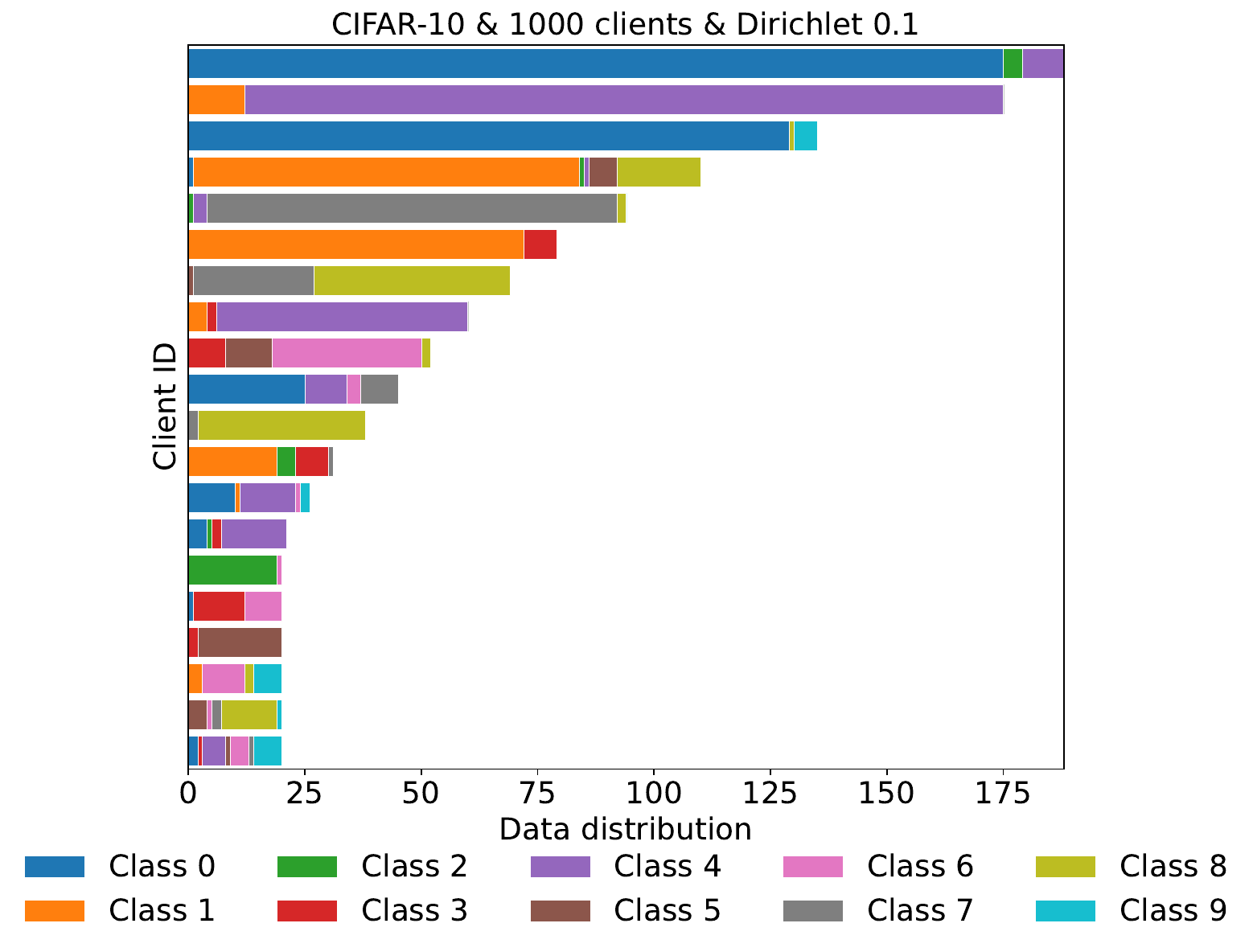}
\caption{Data distribution across clients for CIFAR-10 with 1000 clients under the Dirichlet non-IID setting with $\alpha=0.1$. Each row corresponds to an example client, and the color intensity indicates the proportion of samples from each class. The strong heterogeneity is evident from the concentration of classes within individual clients and from imbalanced data proportions.}
\label{fig:data_distribution}
\end{figure}

\FloatBarrier
\clearpage
\section{Additional results}
\label{app:additional_results}

\subsection{Comparison with history-based conflict mitigation}
\label{app_fedlf_history}

Methods such as FedLF~\citep{pan2024fedlf} mitigate conflicts by explicitly storing and reusing historical client gradients. In contrast, CRAFT avoids maintaining such memory-intensive history buffers. Instead, it carries temporal information through the momentum-guided reference direction $\hat{\mathbf{g}}$, which is reprojected onto the current conflict-free feasible set at each round. Figure~\ref{fig:fedlf_history} shows that CRAFT remains competitive with history-based conflict mitigation while avoiding stale-gradient storage. Therefore, for computational efficiency, we omit the history-based adaptation in all other experiments.

\begin{figure}[ht]
\centering
\includegraphics[width=0.7\linewidth]{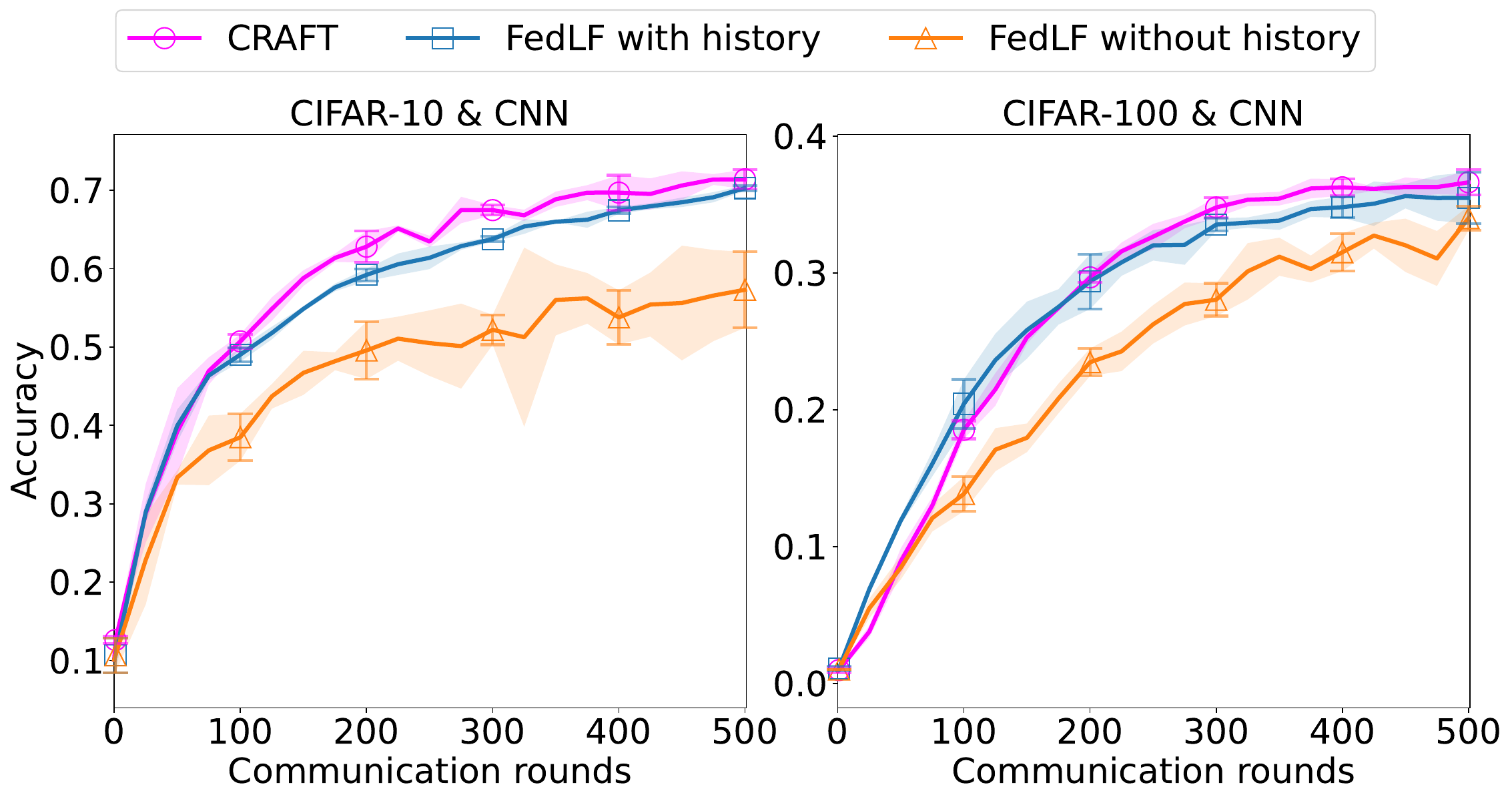}
\caption{Comparison with history-based conflict mitigation methods.}
\label{fig:fedlf_history}
\end{figure}

\subsection{Multiple-seed tests on CIFAR with CNN}
\label{app_multiple_seeds}

We conduct multiple-seed tests to evaluate the stability and robustness of CRAFT compared with baselines. We run the tests on CIFAR-10/100 using CNN models with three random seeds. As shown in Figure~\ref{fig:random_seed_accuracy_evolution}, CRAFT consistently outperforms baselines across all seeds and shows lower variance in the performance metrics. This indicates that CRAFT's improvements are not due to a particular random seed or data split, but reflect a robust enhancement in handling client heterogeneity.

\begin{figure}[ht]
\centering
\includegraphics[width=0.95\linewidth]{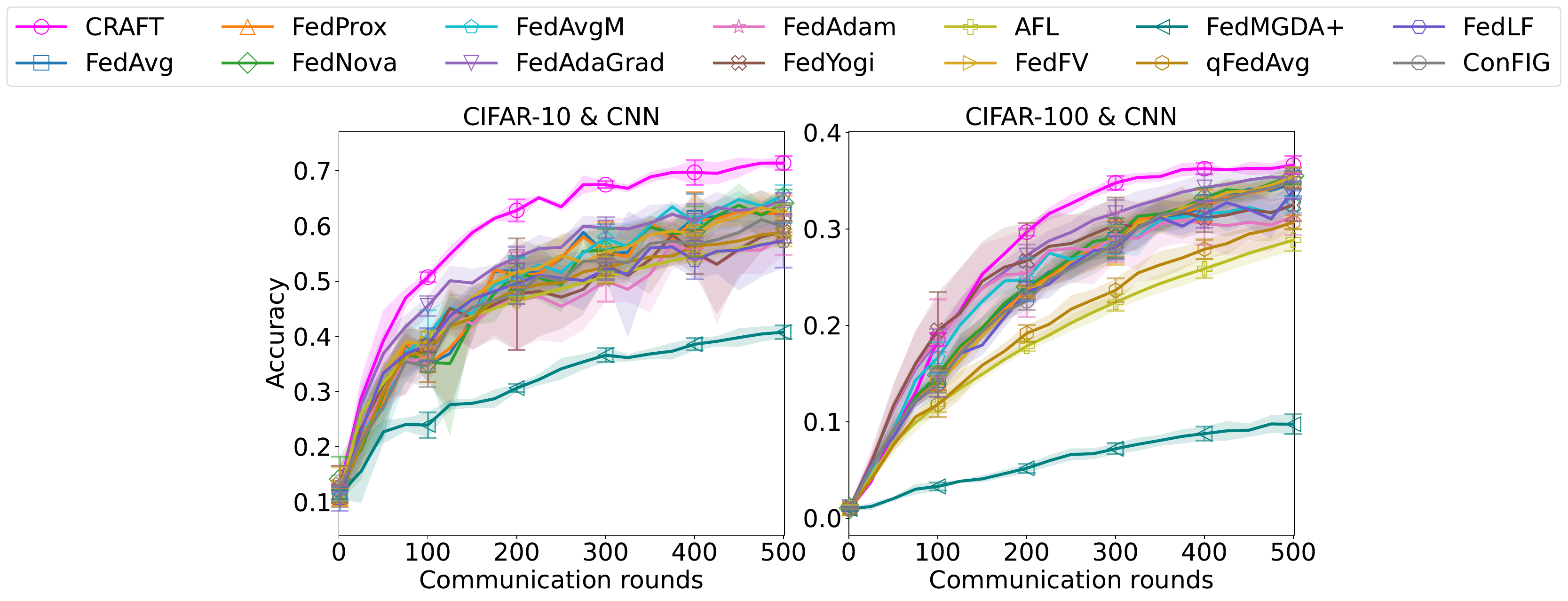}
\caption{Evolution of mean test accuracy over three random seeds. The error bars indicate the standard deviation across seeds. CRAFT consistently performs better.}
\label{fig:random_seed_accuracy_evolution}
\end{figure}

\subsection{Detailed FEMNIST/CIFAR-10/CIFAR-100 results}

We provide more detailed FEMNIST/CIFAR-10/CIFAR-100 results in Table~\ref{tab:main_results_C10_Res}--\ref{tab:main_results_C10_CNN} and Figure~\ref{fig:client_accuracy_distribution_all_methods}. In most settings, CRAFT achieves the strongest overall performance across these metrics.

Figure~\ref{fig:fairness} compares the standard deviation of per-client accuracy on CIFAR-10 with ResNet-56 and ResNet-110. CRAFT achieves the lowest standard deviation in both settings, indicating a tighter distribution of client accuracy.

\begin{figure}[ht]
\centering
\begin{subfigure}{0.49\linewidth}
    \centering
    \includegraphics[width=0.95\linewidth]{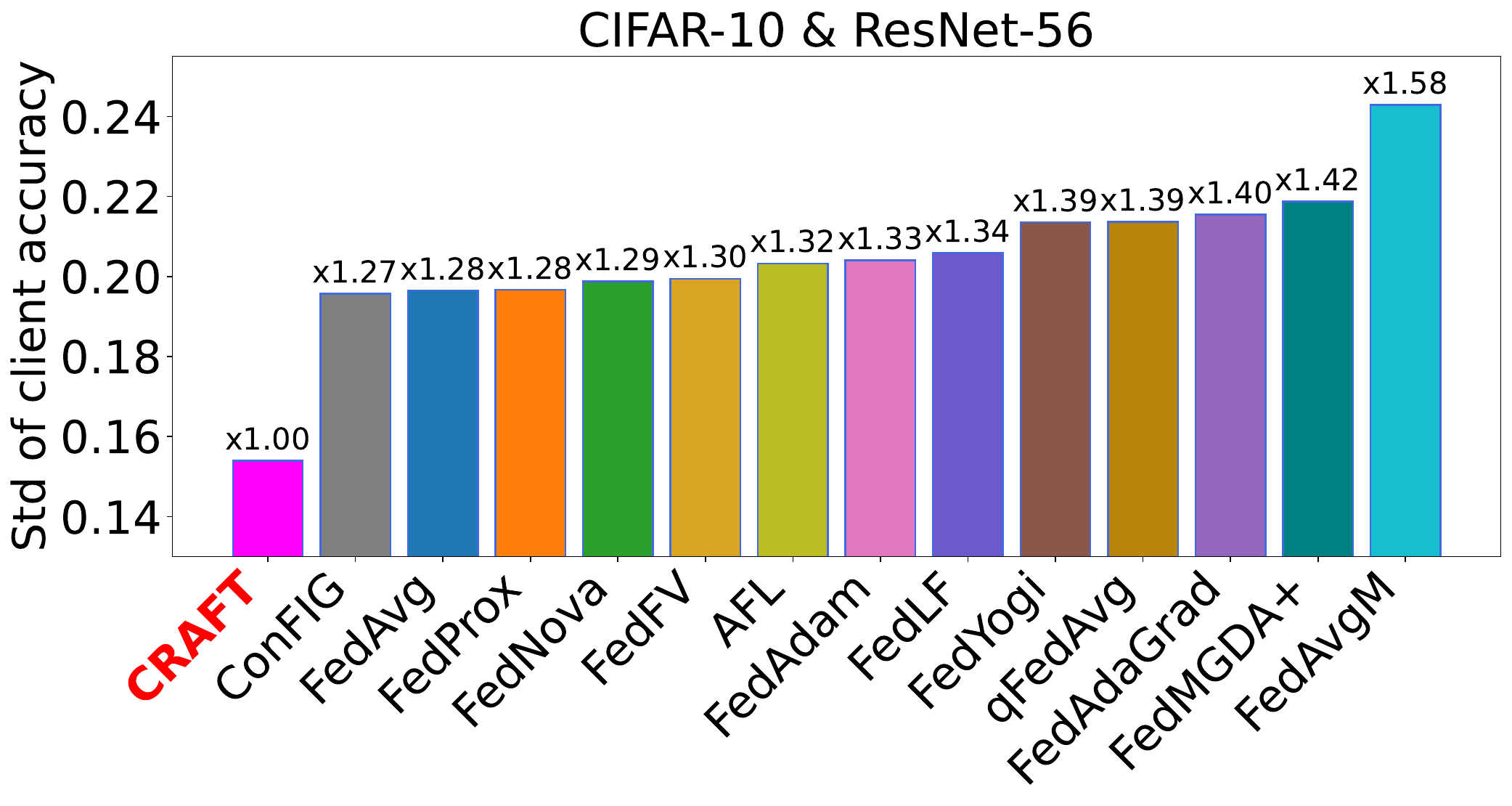}
    \label{fig:fairness_resnet56}
\end{subfigure}
\begin{subfigure}{0.49\linewidth}
    \centering
    \includegraphics[width=0.95\linewidth]{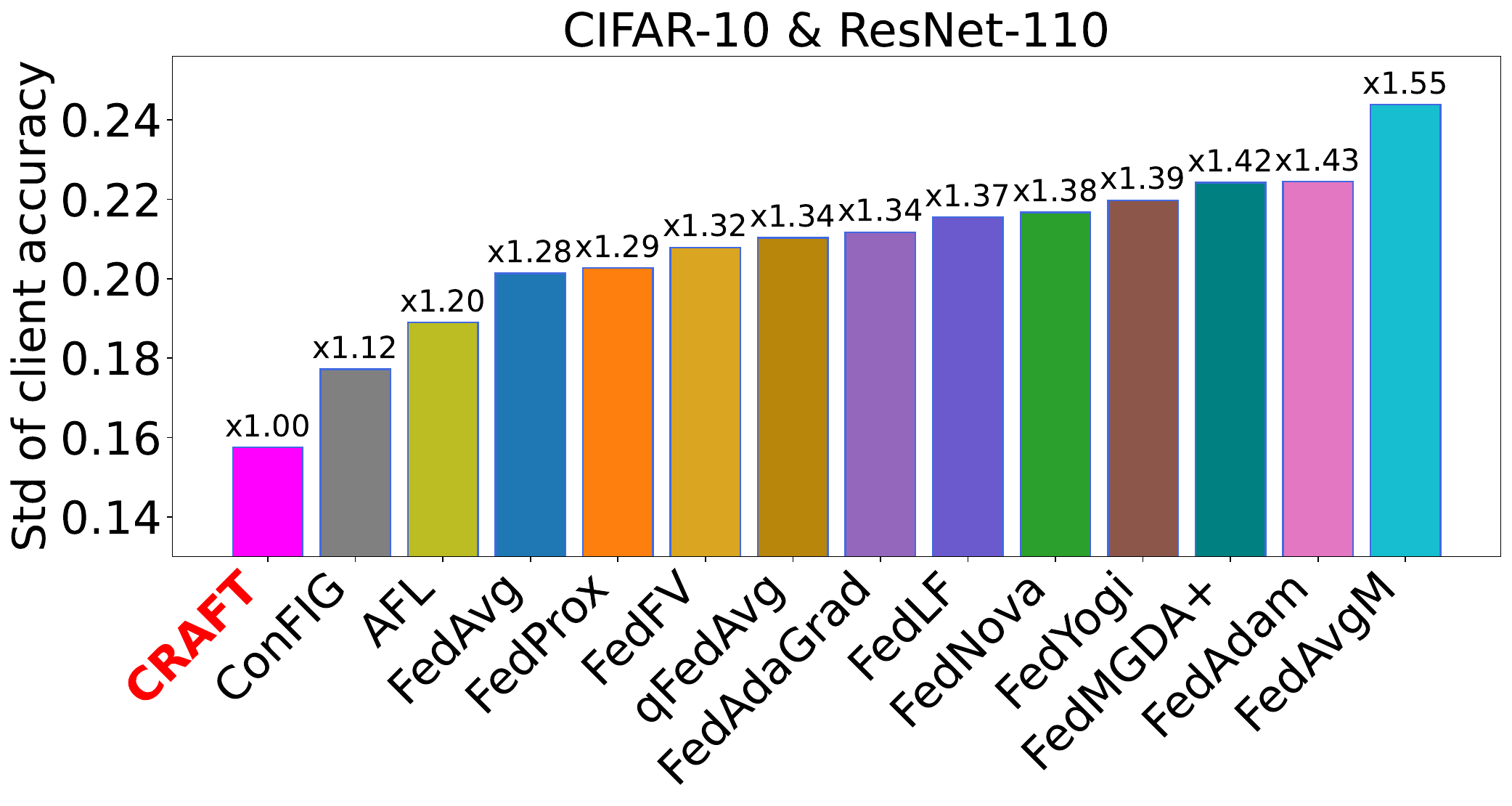}
    \label{fig:fairness_resnet110}
\end{subfigure}
\caption{Client-level fairness with deeper ResNets, measured by the standard deviation of per-client accuracy. Lower values indicate a tighter distribution and therefore better fairness.}
\label{fig:fairness}
\end{figure}

One exception is the CIFAR-100 setting, where CRAFT attains the highest standard deviation across clients. However, this does not indicate weaker robustness. Note that most baselines produce per-client accuracies that are heavily concentrated near zero. In such cases, a low standard deviation can only reflect \emph{uniformly poor} client performance rather than fairness.

In contrast, CRAFT improves accuracy for a larger fraction of clients, thereby shifting the distribution toward higher accuracies while also increasing its spread. This behavior is also illustrated in the last column of Figure~\ref{fig:client_accuracy_distribution_all_methods}. CRAFT produces a more favorable per-client accuracy distribution, with a clear rightward shift and improved tail performance compared with the baselines.

\begin{figure}[ht]
\centering
\includegraphics[width=0.95\linewidth]{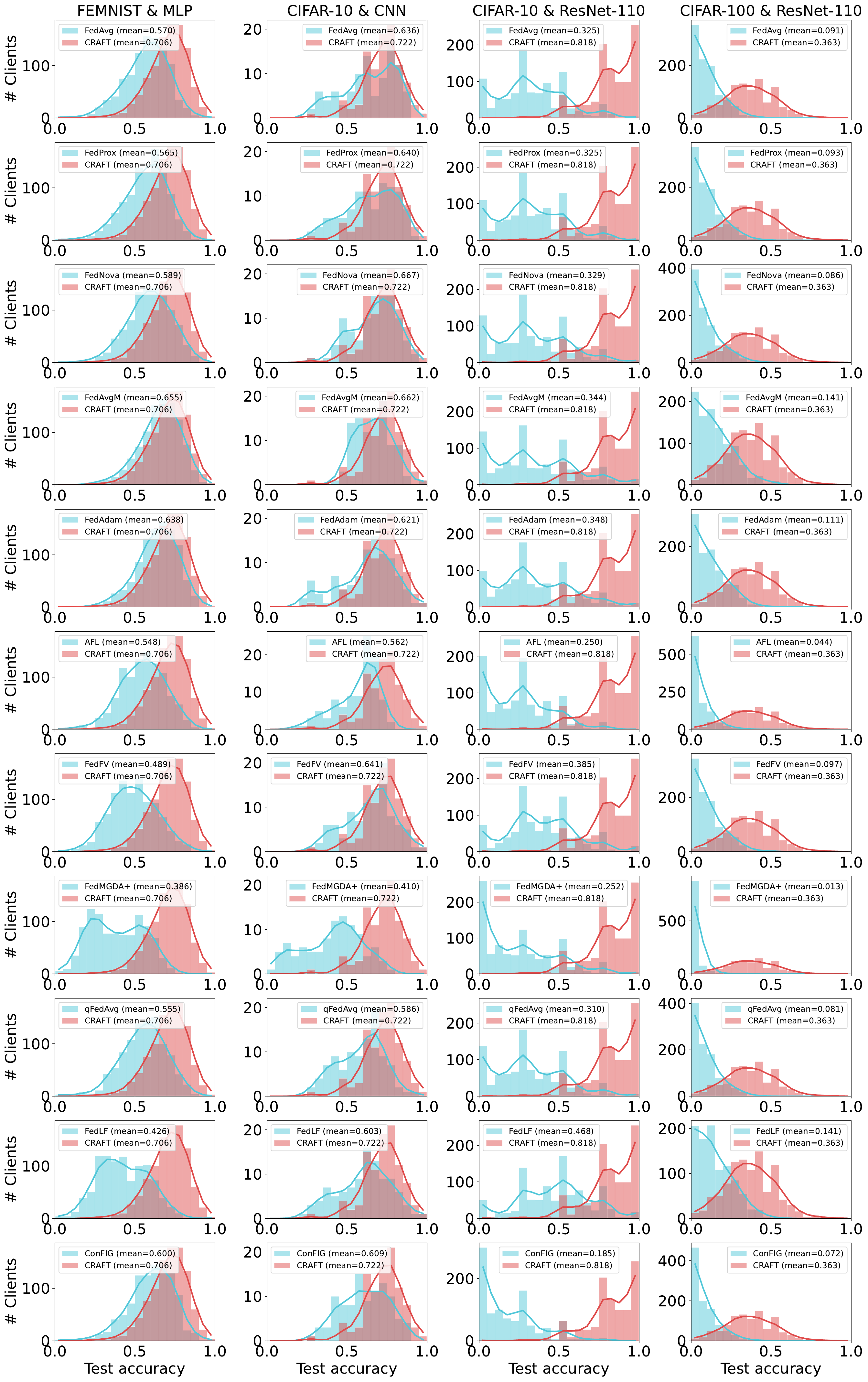}
\caption{Distribution of per-client test accuracy. Each row corresponds to a baseline, and each column corresponds to a test setting. A rightward shift indicates higher accuracy for more clients, and CRAFT produces a more favorable distribution with improved mean and tail performance.}
\label{fig:client_accuracy_distribution_all_methods}
\end{figure}

\begin{table}[htbp]
\centering
\vspace{-5pt}
\caption{Evaluation on CIFAR-10 with different ResNets.}
\label{tab:main_results_C10_Res}
\setlength{\tabcolsep}{1pt}
\resizebox{\textwidth}{!}{
    \begin{tabular}{lcccccccccccc}
        \toprule
        & \multicolumn{4}{c}{CIFAR-10 \& ResNet-20} & \multicolumn{4}{c}{CIFAR-10 \& ResNet-56} & \multicolumn{4}{c}{CIFAR-10 \& ResNet-110} \\
        \cmidrule(lr){2-5} \cmidrule(lr){6-9} \cmidrule(lr){10-13}
        Algorithm & Mean ($\uparrow$) & Best ($\uparrow$) & Worst ($\uparrow$) & Std ($\downarrow$) & Mean ($\uparrow$) & Best ($\uparrow$) & Worst ($\uparrow$) & Std ($\downarrow$) & Mean ($\uparrow$) & Best ($\uparrow$) & Worst ($\uparrow$) & Std ($\downarrow$) \\
        \midrule
        FedAvg & 0.356 & 0.727 & 0.015 & 0.195 & 0.362 & 0.728 & 0.016 & 0.196 & 0.325 & 0.712 & 0.000 & 0.201 \\
        FedProx & 0.356 & 0.728 & 0.015 & 0.195 & 0.363 & 0.721 & 0.012 & 0.197 & 0.325 & 0.713 & 0.000 & 0.203 \\
        FedNova & 0.381 & 0.734 & 0.040 & \underline{0.190} & 0.378 & 0.737 & 0.022 & 0.199 & 0.329 & 0.741 & 0.000 & 0.217 \\
        FedAvgM & 0.463 & 0.871 & 0.040 & 0.241 & 0.399 & 0.826 & 0.009 & 0.243 & 0.344 & 0.811 & 0.000 & 0.244 \\
        FedAdaGrad & 0.427 & 0.799 & 0.072 & 0.206 & 0.401 & 0.792 & 0.026 & 0.215 & 0.361 & 0.767 & 0.005 & 0.212 \\
        FedAdam & 0.412 & 0.790 & 0.065 & 0.197 & 0.376 & 0.760 & 0.020 & 0.204 & 0.348 & 0.785 & 0.001 & 0.224 \\
        FedYogi & 0.398 & 0.777 & 0.047 & 0.201 & 0.368 & 0.771 & 0.010 & 0.214 & 0.336 & 0.760 & 0.000 & 0.220 \\
        AFL & 0.294 & 0.732 & 0.000 & 0.217 & 0.294 & 0.701 & 0.000 & 0.203 & 0.250 & 0.597 & 0.000 & 0.189 \\
        FedFV & 0.421 & 0.780 & 0.078 & 0.194 & 0.416 & 0.793 & 0.076 & 0.199 & 0.385 & 0.779 & 0.029 & 0.208 \\
        FedMGDA+ & 0.235 & 0.612 & 0.000 & 0.193 & 0.249 & 0.693 & 0.000 & 0.219 & 0.252 & 0.693 & 0.000 & 0.224 \\
        qFedAvg & 0.352 & 0.716 & 0.019 & 0.191 & 0.350 & 0.760 & 0.006 & 0.214 & 0.310 & 0.712 & 0.000 & 0.210 \\
        FedLF & 0.501 & 0.845 & \underline{0.149} & 0.208 & \underline{0.494} & \underline{0.831} & \underline{0.132} & 0.206 & \underline{0.468} & \underline{0.822} & \underline{0.082} & 0.215 \\
        ConFIG & \underline{0.525} & \underline{0.872} & 0.146 & 0.211 & 0.317 & 0.694 & 0.000 & \underline{0.196} & 0.185 & 0.556 & 0.000 & \underline{0.177} \\
        \midrule
        \textbf{CRAFT} & \textbf{0.806} & \textbf{1.000} & \textbf{0.479} & \textbf{0.159} & \textbf{0.823} & \textbf{1.000} & \textbf{0.500} & \textbf{0.154} & \textbf{0.818} & \textbf{1.000} & \textbf{0.486} & \textbf{0.158} \\
        \bottomrule
    \end{tabular}
}
\end{table}
\vspace{-4pt}
\begin{table}[htbp]
\centering
\vspace{-4pt}
\caption{Evaluation on CIFAR-100 with different ResNets.}
\label{tab:main_results_C100_Res}
\setlength{\tabcolsep}{1pt}
\resizebox{\textwidth}{!}{
    \begin{tabular}{lcccccccccccc}
        \toprule
        & \multicolumn{4}{c}{CIFAR-100 \& ResNet-20} & \multicolumn{4}{c}{CIFAR-100 \& ResNet-56} & \multicolumn{4}{c}{CIFAR-100 \& ResNet-110} \\
        \cmidrule(lr){2-5} \cmidrule(lr){6-9} \cmidrule(lr){10-13}
        Algorithm & Mean ($\uparrow$) & Best ($\uparrow$) & Worst ($\uparrow$) & Std ($\downarrow$) & Mean ($\uparrow$) & Best ($\uparrow$) & Worst ($\uparrow$) & Std ($\downarrow$) & Mean ($\uparrow$) & Best ($\uparrow$) & Worst ($\uparrow$) & Std ($\downarrow$) \\
        \midrule
        FedAvg & 0.122 & 0.336 & 0.000 & 0.104 & 0.097 & 0.297 & \underline{0.000} & 0.097 & 0.091 & 0.273 & \underline{0.000} & 0.088 \\
        FedProx & 0.122 & 0.336 & 0.000 & 0.104 & 0.097 & 0.298 & \underline{0.000} & 0.097 & 0.093 & 0.281 & \underline{0.000} & 0.090 \\
        FedNova & 0.124 & 0.337 & 0.000 & 0.105 & 0.100 & 0.298 & \underline{0.000} & 0.097 & 0.086 & 0.281 & \underline{0.000} & 0.091 \\
        FedAvgM & \underline{0.213} & \underline{0.462} & \underline{0.008} & 0.130 & 0.166 & 0.398 & \underline{0.000} & 0.117 & \underline{0.141} & \underline{0.372} & \underline{0.000} & 0.114 \\
        FedAdaGrad & 0.178 & 0.433 & 0.000 & 0.127 & 0.132 & 0.352 & \underline{0.000} & 0.108 & 0.121 & 0.328 & \underline{0.000} & 0.103 \\
        FedAdam & 0.174 & 0.415 & 0.000 & 0.121 & 0.123 & 0.335 & \underline{0.000} & 0.105 & 0.111 & 0.314 & \underline{0.000} & 0.101 \\
        FedYogi & 0.172 & 0.408 & 0.000 & 0.120 & 0.116 & 0.327 & \underline{0.000} & 0.104 & 0.078 & 0.265 & \underline{0.000} & 0.087 \\
        AFL & 0.067 & 0.244 & 0.000 & \underline{0.079} & 0.053 & 0.213 & \underline{0.000} & \underline{0.072} & 0.044 & 0.194 & \underline{0.000} & \underline{0.067} \\
        FedFV & 0.139 & 0.352 & 0.000 & 0.110 & 0.120 & 0.329 & \underline{0.000} & 0.104 & 0.097 & 0.295 & \underline{0.000} & 0.095 \\
        FedMGDA+ & 0.029 & 0.164 & 0.000 & \textbf{0.057} & 0.011 & 0.098 & \underline{0.000} & \textbf{0.032} & 0.013 & 0.114 & \underline{0.000} & \textbf{0.040} \\
        qFedAvg & 0.126 & 0.349 & 0.000 & 0.108 & 0.098 & 0.291 & \underline{0.000} & 0.095 & 0.081 & 0.259 & \underline{0.000} & 0.085 \\
        FedLF & 0.206 & 0.443 & 0.000 & 0.127 & 0.149 & 0.383 & \underline{0.000} & 0.113 & 0.141 & 0.360 & \underline{0.000} & 0.110 \\
        ConFIG & 0.208 & 0.458 & 0.000 & 0.131 & \underline{0.181} & \underline{0.418} & \underline{0.000} & 0.121 & 0.072 & 0.250 & \underline{0.000} & 0.084 \\
        \midrule
        \textbf{CRAFT} & \textbf{0.398} & \textbf{0.673} & \textbf{0.137} & 0.154 & \textbf{0.381} & \textbf{0.661} & \textbf{0.112} & 0.156 & \textbf{0.363} & \textbf{0.629} & \textbf{0.110} & 0.149 \\
        \bottomrule
    \end{tabular}
}
\end{table}
\vspace{-4pt}
\begin{table}[htbp]
\centering
\vspace{-4pt}
\caption{Evaluation on FEMNIST and CIFAR-10 with different networks.}
\label{tab:main_results_C10_CNN}
\setlength{\tabcolsep}{1pt}
\resizebox{\textwidth}{!}{
    \begin{tabular}{lcccccccccccc}
        \toprule
        & \multicolumn{4}{c}{FEMNIST \& MLP} & \multicolumn{4}{c}{CIFAR-10 \& CNN} & \multicolumn{4}{c}{CIFAR-100 \& CNN} \\
        \cmidrule(lr){2-5} \cmidrule(lr){6-9} \cmidrule(lr){10-13}
        Algorithm & Mean ($\uparrow$) & Best ($\uparrow$) & Worst ($\uparrow$) & Std ($\downarrow$) & Mean ($\uparrow$) & Best ($\uparrow$) & Worst ($\uparrow$) & Std ($\downarrow$) & Mean ($\uparrow$) & Best ($\uparrow$) & Worst ($\uparrow$) & Std ($\downarrow$) \\
        \midrule
        FedAvg & 0.570 & 0.776 & 0.305 & 0.136 & 0.636 & 0.856 & 0.318 & 0.167 & 0.335 & 0.474 & 0.209 & 0.079 \\
        FedProx & 0.565 & 0.770 & 0.300 & 0.135 & 0.640 & 0.867 & 0.335 & 0.162 & 0.342 & \underline{0.485} & 0.223 & 0.079 \\
        FedNova & 0.589 & 0.804 & 0.337 & 0.135 & \underline{0.667} & 0.867 & 0.413 & 0.138 & 0.345 & \textbf{0.490} & 0.229 & 0.078 \\
        FedAvgM & \underline{0.655} & \underline{0.843} & \underline{0.408} & 0.125 & 0.662 & 0.846 & \underline{0.491} & \textbf{0.108} & 0.343 & 0.446 & 0.201 & 0.070 \\
        FedAdaGrad & 0.645 & 0.828 & 0.405 & \underline{0.121} & 0.648 & 0.850 & 0.472 & 0.116 & 0.340 & 0.464 & 0.236 & 0.068 \\
        FedAdam & 0.638 & 0.825 & 0.390 & 0.125 & 0.621 & 0.872 & 0.272 & 0.173 & 0.314 & 0.456 & 0.207 & 0.069 \\
        FedYogi & 0.639 & 0.830 & 0.381 & 0.130 & 0.607 & 0.865 & 0.345 & 0.152 & 0.333 & 0.453 & \underline{0.237} & 0.063 \\
        AFL & 0.548 & 0.774 & 0.301 & 0.137 & 0.562 & 0.733 & 0.285 & 0.134 & 0.276 & 0.394 & 0.172 & \underline{0.062} \\
        FedFV & 0.489 & 0.742 & 0.246 & 0.144 & 0.641 & \underline{0.875} & 0.361 & 0.150 & \underline{0.346} & 0.474 & 0.225 & 0.074 \\
        FedMGDA+ & 0.386 & 0.658 & 0.133 & 0.166 & 0.410 & 0.690 & 0.083 & 0.182 & 0.090 & 0.197 & 0.017 & \textbf{0.051} \\
        qFedAvg & 0.555 & 0.774 & 0.298 & 0.137 & 0.586 & 0.808 & 0.336 & 0.143 & 0.299 & 0.420 & 0.197 & 0.065 \\
        FedLF & 0.426 & 0.688 & 0.177 & 0.154 & 0.603 & 0.872 & 0.303 & 0.165 & 0.332 & 0.482 & 0.212 & 0.075 \\
        ConFIG & 0.600 & 0.801 & 0.333 & 0.136 & 0.609 & 0.834 & 0.362 & 0.142 & 0.341 & 0.472 & 0.218 & 0.077 \\
        \midrule
        \textbf{CRAFT} & \textbf{0.706} & \textbf{0.878} & \textbf{0.474} & \textbf{0.115} & \textbf{0.722} & \textbf{0.896} & \textbf{0.491} & \underline{0.116} & \textbf{0.359} & 0.484 & \textbf{0.255} & 0.069 \\
        \bottomrule
    \end{tabular}
}
\end{table}

\FloatBarrier
\subsection{Robustness to intermediate data heterogeneity}
\label{app:heterogeneity_robustness}

To assess robustness across different non-IID distributions, we extend the evaluation from the highly heterogeneous Dirichlet setting with $\alpha=0.1$ to milder heterogeneity levels with $\alpha\in\{0.5,1\}$. A larger concentration parameter yields more similar class distributions across clients. We use CIFAR-10 and CIFAR-100 with ResNet-20, $N=1000$ total clients, and $m=100$ active clients per communication round. All methods follow the same evaluation protocol, with the remaining hyperparameters specified in \Cref{app:hyperparameters}.

Table~\ref{tab:heterogeneity_robustness} reports test accuracy at round 500 for CRAFT and four baselines. CRAFT ranks first in all six settings. ConFIG achieves the second-best result in five settings, while FedAvgM ranks second on CIFAR-100 at $\alpha=0.1$. At the milder heterogeneity levels $\alpha=0.5$ and $\alpha=1$, CRAFT maintains an accuracy of 0.812 on CIFAR-10 and achieves 0.418 and 0.421 on CIFAR-100, respectively. These results show that CRAFT's advantage over the baselines persists across the tested non-IID distributions and is not confined to the most heterogeneous regime.

\begin{table}[htbp]
\centering
\caption{Test accuracy under different Dirichlet heterogeneity levels. \textbf{Bold} and \underline{underlined} values indicate the best and second-best results in each column, respectively.}
\label{tab:heterogeneity_robustness}
\begingroup
\small
\setlength{\tabcolsep}{5pt}
\begin{tabular}{lcccccc}
    \toprule
    & \multicolumn{3}{c}{CIFAR-10} & \multicolumn{3}{c}{CIFAR-100} \\
    \cmidrule(lr){2-4} \cmidrule(lr){5-7}
    Algorithm & $\alpha=0.1$ & $\alpha=0.5$ & $\alpha=1$ & $\alpha=0.1$ & $\alpha=0.5$ & $\alpha=1$ \\
    \midrule
    FedAvg & 0.356 & 0.375 & 0.390 & 0.122 & 0.135 & 0.136 \\
    FedNova & 0.381 & 0.376 & 0.385 & 0.124 & 0.136 & 0.137 \\
    FedAvgM & 0.463 & 0.351 & 0.366 & \underline{0.213} & 0.103 & 0.099 \\
    ConFIG & \underline{0.525} & \underline{0.591} & \underline{0.625} & 0.208 & \underline{0.222} & \underline{0.248} \\
    \midrule
    \textbf{CRAFT} & \textbf{0.806} & \textbf{0.812} & \textbf{0.812} & \textbf{0.398} & \textbf{0.418} & \textbf{0.421} \\
    \bottomrule
\end{tabular}
\endgroup
\end{table}

\FloatBarrier
\subsection{Larger-scale image benchmark: Tiny-ImageNet}
\label{app:tiny_imagenet}

We further evaluate CRAFT on Tiny-ImageNet \cite{le2015tiny}, a larger and more challenging natural-image benchmark. We use ResNet-20 and partition the data among $N=1000$ clients using a Dirichlet distribution with concentration parameter $\alpha=0.1$. At each communication round, we sample $m=100$ active clients. All other hyperparameters follow the setup in \Cref{app:hyperparameters}.

Table~\ref{tab:tiny_imagenet} reports test accuracy after 1000 communication rounds for CRAFT, two component variants, and four baselines. CRAFT-no-Layerwise uses the same momentum-like reference without layer-wise adaptation, while CRAFT-no-Momentum uses a zero reference with layer-wise adaptation (ConFIG-Layerwise). CAGrad~\citep{liu2021cagrad} is applied to client updates for server-side aggregation. 

CRAFT achieves the highest test accuracy, exceeding the second-best CRAFT-no-Momentum by 10.4 percentage points. Removing layer-wise adaptation reduces accuracy from 0.325 to 0.061. With a zero reference, layer-wise adaptation also improves accuracy from 0.094 for ConFIG to 0.221 for CRAFT-no-Momentum. These results support the importance of layer-wise adaptation on this larger benchmark and the effectiveness of its combination with the momentum-like reference. Overall, CRAFT remains effective beyond the smaller image benchmarks evaluated in the main experiments.

\begin{table}[htbp]
\centering
\small
\caption{Test accuracy on Tiny-ImageNet with ResNet-20 after 1000 communication rounds. \textbf{Bold} and \underline{underlined} values indicate the best and second-best results, respectively.}
\label{tab:tiny_imagenet}
\begin{tabular}{lc}
    \toprule
    Algorithm & Test accuracy ($\uparrow$) \\
    \midrule
    FedAvg & 0.073 \\
    FedAvgM & 0.027 \\
    ConFIG & 0.094 \\
    CAGrad & 0.085 \\
    CRAFT-no-Layerwise & 0.061 \\
    CRAFT-no-Momentum & \underline{0.221} \\
    \midrule
    \textbf{CRAFT} & \textbf{0.325} \\
    \bottomrule
\end{tabular}
\end{table}

\subsection{Comparison with additional conflict-aware methods}
\label{app:conflict_aware_comparisons}

FedHCA$^2$~\citep{lu2024fedhca2} is a federated multi-task learning framework. Its Hyper Conflict-Averse Aggregation (HCA) module adapts CAGrad's conflict-averse objective~\citep{liu2021cagrad} from task gradients to client updates and computes a shared encoder update. The full framework combines this update with personalization mechanisms to produce client-specific models. We therefore evaluate CRAFT both as an aggregator for a single shared global model and as a replacement aggregation module within an adapted personalized pipeline, similar to the integration of CRAFT with Ditto in Section~\ref{subsec:personalized_integration}.

Both comparisons use CIFAR-10 and CIFAR-100 with ResNet-20, $N=1000$ clients, Dirichlet $\alpha=0.1$, and $m\in\{10,50,100\}$ active clients per round. The tables report accuracy at round 500.

\paragraph{Shared global-model comparison.}
We adapt CAGrad for server-side aggregation by treating active-client updates as objective-specific directions, following the update-space formulation of HCA. This comparison trains one shared model without personalization. Table~\ref{tab:cagrad_client_updates} shows that CRAFT achieves the highest accuracy on both datasets at every participation level, with larger gains over CAGrad at moderate and high participation.

\begin{table}[htbp]
\centering
\small
\setlength{\tabcolsep}{4pt}
\caption{Shared global-model test accuracy: CRAFT, CAGrad applied to client updates, and FedAvg.}
\label{tab:cagrad_client_updates}
\begin{tabular}{lcccccc}
    \toprule
    & \multicolumn{3}{c}{CIFAR-10} & \multicolumn{3}{c}{CIFAR-100} \\
    \cmidrule(lr){2-4} \cmidrule(lr){5-7}
    Algorithm & $m=10$ & $m=50$ & $m=100$ & $m=10$ & $m=50$ & $m=100$ \\
    \midrule
    FedAvg & 0.260 & 0.337 & \underline{0.356} & \underline{0.092} & \underline{0.117} & 0.122 \\
    CAGrad & \underline{0.291} & \underline{0.346} & 0.350 & 0.088 & \underline{0.117} & \underline{0.129} \\
    \midrule
    \textbf{CRAFT} & \textbf{0.417} & \textbf{0.731} & \textbf{0.806} & \textbf{0.188} & \textbf{0.391} & \textbf{0.398} \\
    \bottomrule
\end{tabular}
\end{table}

\FloatBarrier
\paragraph{Personalized comparison within adapted FedHCA$^2$.}
In our CIFAR setting, all clients perform the same classification task with ResNet-20. The decoder cross-attention mechanism for incongruent task-specific decoders in the full FedHCA$^2$ framework is therefore not applicable. We construct an adapted FedHCA$^2$ pipeline that retains client-specific models and Hyper Aggregation Weights. We replace only the HCA aggregation rule with CRAFT, keeping all applicable personalization components fixed.

\begin{table}[htbp]
\centering
\small
\setlength{\tabcolsep}{4pt}
\caption{Personalized test accuracy within the same adapted FedHCA$^2$ pipeline, using either HCA or CRAFT as the aggregation module.}
\label{tab:adapted_fedhca_comparison}
\begin{tabular}{lcccccc}
    \toprule
    & \multicolumn{3}{c}{CIFAR-10} & \multicolumn{3}{c}{CIFAR-100} \\
    \cmidrule(lr){2-4} \cmidrule(lr){5-7}
    Aggregation module & $m=10$ & $m=50$ & $m=100$ & $m=10$ & $m=50$ & $m=100$ \\
    \midrule
    HCA & \underline{0.566} & \underline{0.608} & \underline{0.613} & \underline{0.117} & \underline{0.160} & \underline{0.181} \\
    \midrule
    \textbf{CRAFT} & \textbf{0.576} & \textbf{0.723} & \textbf{0.762} & \textbf{0.132} & \textbf{0.261} & \textbf{0.297} \\
    \bottomrule
\end{tabular}
\end{table}

Table~\ref{tab:adapted_fedhca_comparison} shows that replacing HCA with CRAFT improves personalized accuracy in every setting, with larger gains as participation increases. Together, these comparisons support CRAFT's effectiveness for shared global-model training and its use as an aggregation module in this adapted personalized pipeline, complementing the Ditto integration in Section~\ref{subsec:personalized_integration}.

\FloatBarrier
\subsection{Sensitivity to the alignment target}
\label{app:alignment_target_results}

We compare the data-proportional alignment target discussed in \Cref{app:rho_choice} with the equal target $\rho_i^t=1/m$ for each active client. CRAFT denotes the original data-proportional variant, while CRAFT-Equal uses the equal target. Both variants use CIFAR-10 and CIFAR-100 with ResNet-20, $N=1000$ total clients, Dirichlet heterogeneity $\alpha=0.1$, and $m\in\{10,50,100\}$ active clients per round. Table~\ref{tab:alignment_target_comparison} reports test accuracy at round 500 using the hyperparameter selection protocol in \Cref{app:hyperparameters}.

\begin{table}[htbp]
\centering
\small
\setlength{\tabcolsep}{4pt}
\caption{Round-500 test accuracy with equal and data-proportional alignment targets. \textbf{Bold} and \underline{underlined} values indicate the best and second-best results in each column, respectively.}
\label{tab:alignment_target_comparison}
\begin{tabular}{lcccccc}
    \toprule
    & \multicolumn{3}{c}{CIFAR-10} & \multicolumn{3}{c}{CIFAR-100} \\
    \cmidrule(lr){2-4} \cmidrule(lr){5-7}
    Algorithm & $m=10$ & $m=50$ & $m=100$ & $m=10$ & $m=50$ & $m=100$ \\
    \midrule
    CRAFT-Equal & \textbf{0.448} & \textbf{0.753} & \underline{0.805} & \textbf{0.195} & \textbf{0.394} & \textbf{0.418} \\
    \midrule
    \textbf{CRAFT} & \underline{0.417} & \underline{0.731} & \textbf{0.806} & \underline{0.188} & \underline{0.391} & \underline{0.398} \\
    \bottomrule
\end{tabular}
\end{table}

The equal target achieves higher accuracy in five of the six settings. On CIFAR-10, the improvement is largest under sparse participation ($m=10$), increasing accuracy from 0.417 to 0.448. At $m=100$, the two targets perform nearly identically, with 0.805 for CRAFT-Equal and 0.806 for CRAFT. On CIFAR-100, the equal target improves accuracy at every participation level. These results support equal targeting as a useful alternative, while the original data-proportional CRAFT already outperforms the baselines in Table~\ref{tab:cagrad_client_updates}.

For positive round-wise targets satisfying $\sum_{i\in\mathcal{S}_t}\rho_i^t=1$, the minimum target $\rho_{\min}^t:=\min_{i\in\mathcal{S}_t}\rho_i^t$ is at most $1/m$, with equality only for equal targets. This provides motivation for balanced target allocation in view of the inverse dependence on $\rho_{\min}$ in \Cref{thm:full_multi_step_rate}.

\FloatBarrier
\subsection{Layer-wise grouping ablation}
\label{app:layerwise_grouping}
\label{app:grouping_heterogeneity}

We isolate the effect of grouping granularity in CRAFT on CIFAR-10 and CIFAR-100 using ResNet-20 with $N=1000$ clients. All configurations retain the same momentum-like reference and follow the same training and evaluation protocol, with the remaining settings specified in \Cref{app:hyperparameters}.

We compare five architecture-aligned partitions. With $Q=1$, the full trainable parameter vector forms one group (CRAFT-no-Layerwise). With $Q=2$, the feature extractor and the final classifier form separate groups. With $Q=5$, the stem, the three residual stages, and the classifier each form one group. With $Q=11$, the stem, each of the nine residual blocks, and the classifier form separate groups, with the two convolutions within each block grouped together. The default $Q=20$ separates the stem convolution, the two convolutions in each residual block, and the classifier. In every partition, normalization parameters remain with their associated convolution, as specified in Section~\ref{subsec:layerwise_adaptation}.

The results of varying client participation are shown in Table~\ref{tab:layerwise_grouping} and discussed in Section~\ref{subsec:ablation_studies}. In this part, we focus on the effect of varying data heterogeneity at fixed $m=100$.
Table~\ref{tab:grouping_heterogeneity} shows that $Q=11$ performs best in all four settings with milder heterogeneity ($\alpha\in\{0.5,1\}$). The default $Q=20$ outperforms $Q=1$ throughout, while either $Q=11$ or $Q=20$ achieves the highest accuracy in each column. Together with Table~\ref{tab:layerwise_grouping}, these results show that useful grouping granularity depends on participation and heterogeneity, and that accuracy does not increase monotonically with the number of groups.

We retain the architecture-derived layer-level partition as the default because it provides a simple, deterministic rule without data-dependent tuning. Adaptive grouping would introduce an additional design choice and is left for future work.

\begin{table}[htbp]
\centering
\caption{Layer-wise grouping ablation with varying Dirichlet heterogeneity $\alpha$ at fixed $m=100$: test accuracy at round 500 on CIFAR-10 and CIFAR-100 with ResNet-20 and $N=1000$ clients. All variants retain the momentum-like reference. \textbf{Bold} and \underline{underlined} values indicate the best and second-best results in each column, respectively.}
\label{tab:grouping_heterogeneity}
\begingroup
\small
\setlength{\tabcolsep}{5pt}
\begin{tabular}{lcccccc}
    \toprule
    & \multicolumn{3}{c}{CIFAR-10} & \multicolumn{3}{c}{CIFAR-100} \\
    \cmidrule(lr){2-4} \cmidrule(lr){5-7}
    Configuration & $\alpha=0.1$ & $\alpha=0.5$ & $\alpha=1$ & $\alpha=0.1$ & $\alpha=0.5$ & $\alpha=1$ \\
    \midrule
    CRAFT ($Q=1$) & 0.462 & 0.712 & 0.724 & 0.152 & 0.257 & 0.283 \\
    CRAFT ($Q=2$) & 0.530 & 0.722 & 0.707 & 0.312 & 0.345 & 0.350 \\
    CRAFT ($Q=5$) & 0.754 & 0.791 & 0.796 & 0.394 & \underline{0.421} & \underline{0.425} \\
    CRAFT ($Q=11$) & \underline{0.798} & \textbf{0.833} & \textbf{0.825} & \textbf{0.415} & \textbf{0.424} & \textbf{0.429} \\
    \midrule
    \textbf{CRAFT ($Q=20$)} & \textbf{0.806} & \underline{0.812} & \underline{0.812} & \underline{0.398} & 0.418 & 0.421 \\
    \bottomrule
\end{tabular}
\endgroup
\end{table}

\end{document}